\documentclass{article}

\makeatletter
\def\input@path{{styles/}}
\makeatother
\usepackage[preprint]{colm2026_conference}
\usepackage{fontspec}

\normalfont
\usepackage{amsthm}
\usepackage{marvosym}
\usepackage{microtype}
\usepackage{graphicx}
\usepackage{trimclip}
\usepackage{xcolor}
\usepackage{booktabs}
\usepackage{array}
\usepackage{colortbl}
\usepackage{float}
\usepackage{thmtools}
\usepackage{tikz}
\usepackage{tcolorbox}
\usepackage{pgfplots}
\pgfplotsset{compat=1.18}
\usepgfplotslibrary{groupplots}
\usepackage{hyperref}
\usepackage{amsmath}
\usepackage{url}
\usepackage{graphicx} 
\usepackage{float} 
\usepackage{subfigure} 
\usepackage{booktabs}
\usepackage{algorithm}
\usepackage{algpseudocode}
\usepackage{tablefootnote}
\usepackage{amssymb}

\definecolor{abyss}{HTML}{121D36}
\definecolor{polarnight}{HTML}{1A2947}
\definecolor{nebula}{HTML}{2B3F66}
\definecolor{steeltrail}{HTML}{6D87BD}
\definecolor{skytrail}{HTML}{8FA8D8}
\definecolor{starlight}{HTML}{DFE7F5}
\definecolor{warmstar}{HTML}{E8D9C4}
\definecolor{allsparkwordmark}{HTML}{16233F}
\definecolor{allsparkspark}{HTML}{4A659C}
\definecolor{electricblue}{HTML}{3866FF}
\definecolor{covercream}{HTML}{EEF3FA}
\definecolor{coveraccent}{HTML}{3866FF}
\colorlet{pevekpurple}{skytrail}
\colorlet{bargray}{steeltrail}
\colorlet{barlgray}{starlight}

\newfontfamily\outfit[
  Path=assets/fonts/,
  UprightFont=Outfit-Regular.ttf,
  BoldFont=Outfit-SemiBold.ttf
]{Outfit}

\hypersetup{
  colorlinks=true,
  linkcolor=electricblue,
  citecolor=electricblue,
  urlcolor=coveraccent,
  filecolor=electricblue
}
\setcitestyle{numbers,square,comma,sort&compress}

\newcommand{\reporttitle}{\textcolor{electricblue}{PACT}: From Credit Assignment to Critic Alignment 
\par}
\title{\reporttitle}
\author{AllSpark Team}

\begin{document}

\fancyhead{}
\renewcommand{\headrulewidth}{0pt}
\color{abyss}
\thispagestyle{empty}

\vspace*{-0.44in}
\begin{tcolorbox}[
  width=\linewidth,
  colback=covercream,
  colframe=covercream,
  boxrule=0pt,
  arc=14pt,
  outer arc=14pt,
  boxsep=0pt,
  left=20pt,
  right=20pt,
  top=13pt,
  bottom=11pt
]
  {\outfit\fontsize{17}{19}\selectfont\bfseries\centering
    \textcolor{coveraccent}{}\hspace{0.25em}\reporttitle\par}
  \vspace{1.45em}
{\bfseries\centering AllSpark Team\par}
\vspace{0.8em}

{\normalfont\itshape\color{steeltrail}
\fontsize{11}{14}\selectfont
\setlength{\parskip}{0pt}
\centering
``The real justification of these definitions, however,
will reside in their implications.''\par
}
\nocite{shannon1948mathematical}

\vspace{0.4em}
{\normalfont\fontsize{9}{11}\selectfont
\color{steeltrail}\centering
\setlength{\parskip}{0pt}
--- Claude E. Shannon\par
}
\vspace{1em}
  \begingroup
  \normalfont
  \setlength{\parindent}{0pt}
  \setlength{\parskip}{0pt}
  Reinforcement learning has become a central component of large language model (LLM) post-training, yet token-level credit lacks a generally accepted mathematical definition, leaving its relationship to commonly used training signals unclear. We formulate three regularity conditions, namely Completeness, Prefix Consistency, and Neutrality, and prove that they uniquely determine token-level credit. This characterization provides a unified basis for explaining
phenomena across existing algorithms and guides the development
of an improved actor-critic training procedure. Through this lens, an ideal teacher
in On-Policy Distillation (OPD) acts as an implicit critic,
yielding an expected policy gradient proportional to that
induced by token-level credit. Response-level REINFORCE
Leave-One-Out (RLOO) signals match the expected policy-gradient
contribution of token-level credit despite their coarser
granularity. We further establish approximate credit sparsity
under bounded outcome rewards and show how intermediate critic
errors in Generalized Advantage Estimation (GAE) can become
comparable to the underlying credit. These motivate \textcolor{electricblue}{\textbf{P}}olicy \textcolor{electricblue}{\textbf{A}}ligned \textcolor{electricblue}{\textbf{C}}ritic \textcolor{electricblue}{\textbf{T}}raining (PACT),
which adopts an Actor-then-Critic update order to apply
importance sampling correction to critic training and better
align the critic with the updated policy. In agentic mathematical reasoning, PACT achieves 72.87\% average accuracy across four benchmarks, outperforming GRPO and PPO by 8.80 and 13.16 percentage points, respectively. On SWE-bench Verified, PACT achieves a pass rate of 67.4\%,
outperforming PPO, GRPO, and SAO by 2.4, 2.0, and 3.8 percentage
points, respectively.

  \par
  \endgroup

  \vspace{0.65em}
  \noindent
  \begin{minipage}[b]{0.63\linewidth}
    \outfit\fontsize{8.4}{10.2}\selectfont
    \textbf{Date:} September 22, 2026\\[-0.1em]
    \textbf{Github:}
    \href{https://github.com/AllSpark-Research/PACT}{https://github.com/AllSpark-Research/PACT}
  \end{minipage}%
  \hfill
  \begin{minipage}[b]{0.33\linewidth}
    \raggedleft
    \raisebox{-0.30em}{\includegraphics[height=16pt]{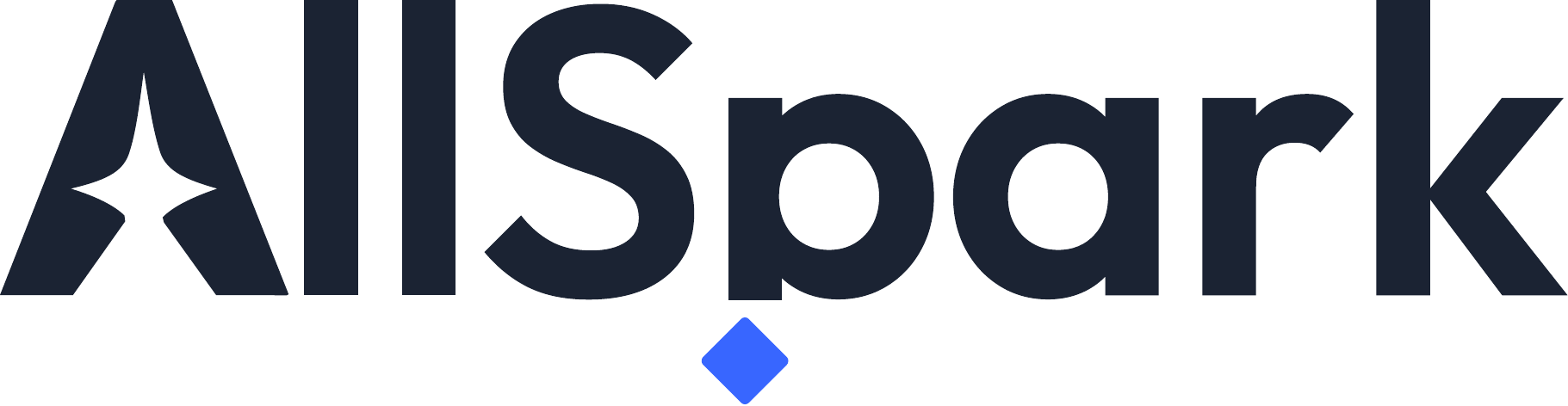}}%
  \end{minipage}
\end{tcolorbox}

\newtheorem{theorem}{Theorem}
\newtheorem{corollary}{Corollary}
\newtheorem{lemma}{Lemma}
\newtheorem{definition}{Definition}
\newtheorem{proposition}{Proposition}

\begin{figure}[h]
    \centering
    \includegraphics[width=\linewidth]
    {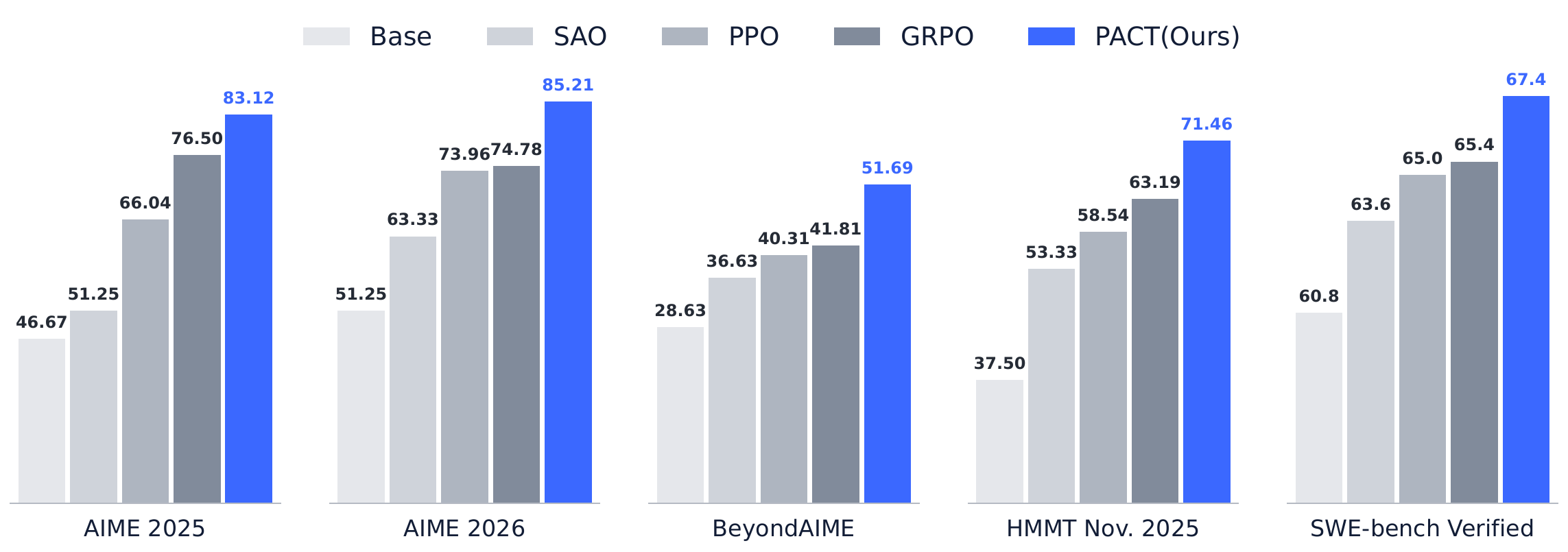}
    \caption{
        Performance comparison of the base models and models trained with
PPO ($\lambda=1$), GRPO, SAO, and PACT on mathematical reasoning
and coding benchmarks. Mathematical reasoning results report
Avg@16 accuracy on AIME 2025, AIME 2026, BeyondAIME, and
HMMT Nov.\ 2025 using Qwen3.5-4B with OpenCode.
Coding results report pass@1 rates on SWE-bench Verified using
Qwen3.6-35B-A3B with Codex.
    }
    \label{fig:main_results}
\end{figure}

\clearpage

\section{Introduction}

Recent advances have demonstrated that reinforcement learning can
substantially improve the capabilities of large language models
across general tasks~\citep{guo2025deepseek, ma2026general,cheng2026revisiting},
general agents~\citep{wang2025ragen},
and coding agents~\citep{da2025agent,wei2025toward}. In these settings, a language model may generate a
long sequence of tokens and repeatedly interact with an external environment.
However, the training signal is often a scalar outcome reward revealed only
after the completion of the trajectory, while optimization updates are applied
at the level of individual generated tokens. This difference in granularity gives rise to a fundamental credit assignment
problem of determining how the final outcome should be attributed to the tokens
in the trajectory. Furthermore, despite its central role in reinforcement
learning, credit itself lacks a generally accepted mathematical
characterization, and existing methods operationalize it through different
algorithm-dependent quantities~\citep{pignatelli2023survey}.

A common formulation models autoregressive language model reinforcement learning as a token-level Markov Decision Process, where the complete generation history is treated as the state and the next generated token as the action~\citep{ramamurthy2022reinforcement}. In interactive settings, taking the complete history of the prompt,
generated tokens, and environment observations as the state yields
a Markov representation. However, this representation alone does not
characterize how the realized outcome should be attributed across
the generated tokens. In particular, it does not describe how the statistical information relevant to the final reward changes as the trajectory unfolds. Credit assignment therefore requires an additional characterization beyond the Markov formulation.

In this work, we formulate three regularity conditions for credit
assignment, namely Completeness, Prefix Consistency, and Neutrality,
and prove that credit is uniquely determined under these conditions.
The resulting representation subsumes related forms previously
derived under specific objectives or algorithmic
constructions~\citep{arjona2019rudder,kazemnejad2024vineppo,setlur2025rewarding}. The same uniqueness result holds for credit assignment at coarser
granularities, such as the turn level, with the corresponding
representation obtained by aggregating token-level credit over
consecutive segments.

Our contributions are twofold:
\begin{itemize}
    \item \textbf{A unique representation theorem and its consequences.}
    We establish a unique representation theorem for credit under three
    regularity conditions, namely Completeness, Prefix Consistency,
    and Neutrality. This representation provides a unified perspective
    on several phenomena in LLM reinforcement learning. 
    
    Under an ideal
    teacher, the On-Policy Distillation (OPD)~\citep{lu2025onpolicydistillation} update is equivalent in expectation, up to a scaling
    factor, to the policy-gradient update induced by the unique credit,
    with the teacher acting as an implicit critic.

    For REINFORCE Leave-One-Out (RLOO)~\citep{ahmadian2024back}, the
    response-level signal, although not itself token-level credit,
    induces the same expected policy-gradient contribution as the unique
    credit representation.

    We further establish approximate sparsity of
    token-level credit under bounded outcome rewards, providing a
    perspective on the sensitivity of fine-grained credit estimation
    to critic errors and the empirical difficulty of Generalized Advantage Estimation (GAE) based
    Actor-Critic methods in long-horizon settings~\citep{schulman2015high}.

    \item \textbf{Policy Aligned Critic Training.}
Motivated by these findings, we propose Policy Aligned Critic
Training (PACT). PACT adopts an Actor-then-Critic update order
to enable importance sampling correction in critic training,
better aligning the critic with the updated policy. We also use Binary Cross Entropy (BCE)
instead of Mean Squared Error (MSE) to train the critic to better approximate the true
values.

On four mathematical reasoning benchmarks, PACT achieves
72.87\% average accuracy, outperforming GRPO and PPO by 8.80 and
13.16 percentage points, respectively. On SWE-bench Verified,
PACT achieves a pass rate of 67.4\%, outperforming PPO,
GRPO, and SAO by 2.4, 2.0, and 3.8 percentage points, respectively.
\end{itemize}

All proofs are deferred to the appendices, which also contain
additional discussion and a detailed review of related work.
\section{Unique Representation of Token-Level Credit}
\label{sec2}
Credit assignment can be defined at different granularities. For example,
one may assign rewards to complete responses, interaction turns, or individual
tokens. In autoregressive language model reinforcement learning, however,
coarser-grained credit assignments can be viewed as special cases of
token-level credit by aggregating consecutive tokens into larger units. Therefore, we shall firstly focus on the finest-grained formulation and study token-level
credit assignment.

\subsection{Notation: Reward Information Flow}
We first introduce the necessary notation to formalize the information flow of
rewards during autoregressive generation. Our formulation covers both standard
autoregressive generation and agentic interaction. Given a prompt \(q\), a realized trajectory
takes the form
\[
Y=(q,T_1,O_1,T_2,O_2,\ldots,T_\tau,O_\tau),
\]
where \(T_i\) is the \(i\)-th token generated by the policy and \(O_i\) is the
observation returned by the environment after \(T_i\). An observation may be a
tool response, an environment transition, user feedback, or any other
information revealed to the model. We allow \(O_i=\varnothing\); consequently,
ordinary autoregressive language generation is recovered as the special case
in which every observation is empty. The terminal reward is defined as a measurable function of the complete trajectory,
\(R=\mathcal{R}(Y).\)
In particular, \(R\) need not be assigned to any specific token or
observation. 

Let
\(
\mathcal{F}_0=\sigma(q),
\mathcal{F}_i
=
\sigma(q,T_1,O_1,\ldots,T_i,O_i)
\)
denote the information available after the \(i\)-th token and its associated
observation have been revealed. Here, $\sigma(\cdot)$ denotes the generated sigma-algebra. We define the conditional reward prediction
at this point as
\(
V_i=\mathbb{E}[R\mid\mathcal{F}_i].
\)

\subsection{Regularity Conditions of Credit Assignment}

Let \(C_i\) denote the credit assigned to the \(i\)-th generated token
\(T_i\). Although environmental observations \(O_i\) are incorporated into the
information filtration and affect future token generation, credit assignment
aims to attribute the final outcome only to the model's generated tokens,
rather than to environment-provided information. Here we also define the accumulated token
credit up to step \(i\) as
\[
S_i=\sum_{j=1}^{i}C_j,\qquad S_0=0.
\]
We characterize token-level credit assignments through the following three
regularity conditions.

\paragraph{Completeness.}
The assigned credits should fully explain the deviation of the final reward
from its initial prediction:
\[
\sum_{i=1}^{\tau}C_i
=
R-\mathbb{E}[R\mid\mathcal F_0].
\]

\paragraph{Prefix Consistency.}
Credit assigned to a generated prefix should only reflect the information
contained in that prefix. In particular, once a prefix has been generated, the
credit accumulated by this prefix should be fixed from the perspective of the
available information. Different realizations of future tokens or observations
should be attributed to future decisions, rather than modifying the credit
already assigned to previous tokens.
Formally, for any two trajectories \(Y\) and \(\widetilde{Y}\), if they share
the same prefix information up to step \(i\),
\[
(q,T_1,O_1,\ldots,T_i,O_i)
=
(\widetilde{q},\widetilde{T}_1,\widetilde{O}_1,\ldots,
\widetilde{T}_i,\widetilde{O}_i),
\]
then the accumulated credits at step \(i\) should satisfy
\(
S_i(Y)=S_i(\widetilde{Y}).
\) 
\paragraph{Neutrality.}
A token credit should neither systematically overestimate nor underestimate
the contribution of the corresponding token from the perspective of the
available information. The expected credit assigned to the next token should be zero:
\[
\mathbb E[C_i\mid\mathcal F_{i-1}]=0.
\]

\subsection{Unique Representation Theorem}

The following theorem is the central result of our characterization and provides the foundation for the subsequent analysis. It establishes that a token-level credit assignment satisfying the three regularity conditions introduced above exists and is unique, and that the unique assignment is given by the differences of a martingale. We further establish in Appendix~\ref{app:necessity_conditions} that all three conditions are necessary for uniqueness by showing that omitting any one admits alternative credit assignments.

\begin{restatable}[Unique Representation of Token-Level Credit]
    {theorem}{uniquecredit}
\label{thm:unique_credit}
Let \(L\) denote the maximum number of generated tokens in a trajectory, and let \(\tau\le L\) be the stopping time induced by the generation process. Assume that \(R\) is integrable and \(\mathcal F_\tau\)-measurable. Define
\[
V_i:=\mathbb E[R\mid\mathcal F_i],
\qquad i=0,\ldots,\tau.
\]
There exists a unique integrable token-level credit assignment,
up to almost-sure equality,
\(\{C_i\}_{i=1}^{\tau}\) satisfying Completeness, Prefix Consistency, and
Neutrality. 

It admits the representation
\[
C_i
=
V_i-V_{i-1}
=
\mathbb E[R\mid\mathcal F_i]
-
\mathbb E[R\mid\mathcal F_{i-1}]
\quad \text{a.s.},
\qquad i=1,\ldots,\tau.
\]
Moreover, \(\{C_i\}_{i=1}^{\tau}\) is a martingale difference sequence with
respect to \(\{\mathcal F_i\}_{i=0}^{\tau}\).
\end{restatable}

\begin{restatable}[Aggregation of the Unique Token-Level Representation]{corollary}{Cor}
\label{cor:coarse_credit}
Consider any partition of the token sequence into consecutive segments
\[
\mathcal I_1,\mathcal I_2,\ldots,\mathcal I_K,
\]
where each segment \(\mathcal I_k\) contains a set of consecutive token
indices. Define the segment-level credit as
\[
C^{(k)}
=
\sum_{i\in\mathcal I_k} C_i .
\]
The resulting segment-level credit is
\[
C^{(k)}
=
\sum_{i\in\mathcal I_k}
\left(
V_i-V_{i-1}
\right)
=
V_{b_k}-V_{a_k-1},
\]
where \(a_k\) and \(b_k\) denote the first and last token indices of segment
\(\mathcal I_k\).

\end{restatable}

In particular, existing coarse-grained credit assignment schemes, such as
turn-level attribution, can be recovered as special cases by
aggregating the unique token-level representation over appropriate token
segments.

\section{Interpreting Existing RL Algorithms through the Unique Credit Representation}

In this section, we use the unique credit representation to revisit
several phenomena in LLM reinforcement learning. We first examine
the relationship between OPD and critics, then study credit
granularity in RLOO and credit estimation in GAE. All these analyses together motivate Policy Aligned Critic Training (PACT), introduced in
the next section.

\subsection{The Teacher in On-Policy Distillation is an Implicit Critic}
\label{sec:opd_credit}

On-Policy Distillation trains on student-generated trajectories
using dense token-level supervision from a teacher
\citep{lu2025onpolicydistillation}. Prior work further showed that the
OPD objective can be algebraically reformulated as dense
KL-constrained RL, in which the teacher--student log-density ratio
plays the role of a token-level advantage \citep{yang2026learning}.
This optimization-level equivalence, however, does not by itself
guarantee that the induced signal is consistent with the outcome reward
\(R\), or that it constitutes a correct token-level credit assignment.
We establish this missing credit-level connection by identifying an
ideal teacher under which OPD induces, up to scaling, the same
policy-gradient update as the unique token-level credit in
Theorem~\ref{thm:unique_credit}.

Moreover, recent work suggests that effective on-policy distillation requires a teacher that improves upon the student while remaining compatible with its on-policy distribution \citep{li2026rethinking}. We capture these requirements with an idealized teacher defined through KL-regularized reward improvement. 

Fix a policy \(\pi=\pi_\theta\) and a token position \(t\leq\tau\). Let
\[
p_t(a)
=
\pi_\theta(a\mid\mathcal F_{t-1})
\]
denote the current policy distribution over the vocabulary
\(\mathcal V\). Define the token-action value
\[
Q_t^\pi(a)
=
\mathbb E_\pi
\left[
R
\middle|
\mathcal F_{t-1},T_t=a
\right].
\]
The expectation includes the subsequent environment observation and all
future tokens and observations generated under \(\pi\). 

\begin{definition}[Ideal Teacher]
\label{def:ideal_opd_teacher}
Let \(\beta>0\). At each prefix \(\mathcal F_{t-1}\), the ideal teacher
distribution is defined by
\[
q_t^\star
\in
\arg\max_{q\in\Delta(\mathcal V)}
\left\{
\mathbb E_{a\sim q}
\left[
Q_t^\pi(a)
\right]
-
\beta
D_{\mathrm{KL}}(q\|p_t)
\right\}.
\]
We assume that \(R\) is bounded and that \(p_t\) has full support over
\(\mathcal V\). The distribution \(q_t^\star\) is defined relative to the
current policy and is held fixed during the corresponding OPD update.
\end{definition}

For a teacher distribution \(q_t\), sampled-token OPD assigns the signal
\[
A_t^{\mathrm{OPD}}(q_t)
=
\log
\frac{q_t(T_t)}{p_t(T_t)},
\qquad
T_t\sim p_t.
\]
Let
\[
Z_t
=
\nabla_\theta
\log\pi_\theta(T_t\mid\mathcal F_{t-1})
\]
denote the policy score. The on-policy update direction induced by
\(A_t^{\mathrm{OPD}}(q_t)\) is
\[
\begin{aligned}
G_t^{\mathrm{OPD}}(q_t)
:=
\mathbb E_\pi
\left[
Z_t A_t^{\mathrm{OPD}}(q_t)
\middle|
\mathcal F_{t-1}
\right] 
=
\mathbb E_\pi
\left[
Z_t
\log\frac{q_t(T_t)}{p_t(T_t)}
\middle|
\mathcal F_{t-1}
\right].
\end{aligned}
\]

The following theorem identifies the OPD teacher as an implicit critic
under the ideal-teacher assumption. Unlike a conventional critic
represented by a value head, the OPD teacher provides credit through
its KL-based objective. From this perspective, the empirical success of OPD further
highlights the role of critics, even when no explicit value head
is used.

\begin{restatable}[Teacher Credit Equivalence]{theorem}{opd}
\label{prop:opd_credit_equivalence}
Let \(q_t^\star\) be the ideal teacher in
Definition~\ref{def:ideal_opd_teacher}. Then
\[
G_t^{\mathrm{OPD}}(q_t^\star)
=
\frac{1}{\beta}
\mathbb E_\pi
\left[
Z_t C_t^\pi
\middle|
\mathcal F_{t-1}
\right]
\qquad\text{a.s.},
\]
where
\[
C_t^\pi
=
V_t^\pi-V_{t-1}^\pi,
\qquad
V_t^\pi
=
\mathbb E_\pi[R\mid\mathcal F_t]
\]
is the unique token-level credit from
Theorem~\ref{thm:unique_credit}.
\end{restatable}

\subsection{Group-Level Baselines as Gradient-Equivalent Credit Surrogates}

The unique representation theorem shows that the unique token-level credit
depends on the evolution of conditional reward predictions:
\(
C_i=V_i-V_{i-1}.
\)
However, recent group-based RL algorithms for language models, such as GRPO~\citep{shao2024deepseekmath} and RLOO~\citep{ahmadian2024back}, construct
advantages using multiple sampled responses from the same prompt. The resulting
baseline is defined at the response level, since it aggregates rewards from
different sampled trajectories. This response-level baseline is then
broadcast to every token position within the same response during policy
optimization. This introduces a granularity gap between response-level baseline estimation
and token-level credit assignment.

We now examine group-level baseline methods, taking RLOO as a representative
example. We show that, although the RLOO baseline is constructed at the
response level, the resulting token-wise policy gradient contribution
coincides in expectation with that of the unique token-level credit.

\begin{theorem}[Gradient Unbiasedness of the RLOO Estimator]
\label{thm:rloo_unbiased}

Conditioned on a prompt \(q\), let \(G\geq2\) trajectories be
sampled independently from the current policy \(\pi_\theta\).
Let \(R_j\) denote the reward of trajectory \(j\).
For trajectory \(i\), define the leave-one-out baseline
\[
\bar R_{-i}
=
\frac{1}{G-1}\sum_{j\neq i}R_j .
\]
Then its token-level policy gradient estimator is equivalent in expectation
to the gradient induced by the unique token-level credit representation:
\[
\mathbb E\left[
\nabla_\theta\log\pi_\theta(T_t|\mathcal F_{t-1})
(R_i-\bar R_{-i})
\right]
=
\mathbb E\left[
\nabla_\theta\log\pi_\theta(T_t|\mathcal F_{t-1})
C_t
\right].
\]
\end{theorem}

The preceding equivalence holds only in expectation and does not
imply equal statistical efficiency. The conditional expectation
$V_t=\mathbb{E}[R\mid\mathcal{F}_t]$ minimizes the mean squared
prediction error among $\mathcal{F}_t$-measurable predictors.
In contrast, the response-level RLOO baseline retains randomness
from finite outcomes. In long-horizon settings, these fluctuations
can be substantial relative to the local credit signal. This
motivates estimating credit at a finer granularity.

\subsection{Approximate Credit Sparsity and Error Sensitivity of GAE}

Recent empirical observations suggest that GAE with
\(\lambda\) close to one can be beneficial in LLM reinforcement
learning. DeepSeek-R1~\citep{guo2025deepseek} reports that
PPO~\citep{schulman2017proximal} with \(\lambda=1\)
outperforms the commonly used setting \(\lambda=0.95\).
SAO~\citep{hou2026single} adopts a length-adaptive GAE
coefficient that approaches one as the response length increases.
To understand these observations, we first establish an
approximate sparsity property of credit and then examine
how critic estimation errors enter GAE.

For any bounded reward, we can apply a positive affine transformation to
normalize it into \([0,1]\) without changing the optimal policy. Therefore, we
consider \(R\in[0,1]\) without loss of generality. In this case, the reward
variance satisfies
\(
\operatorname{Var}(R)\leq\frac14 .
\)

\begin{restatable}[Approximate Credit Sparsity]{theorem}{ACS}
\label{lem:credit_dilution}
For any fixed prompt \(q\),
\[
\mathbb E \left[\sum_{i=1}^{\tau}C_i^2\mid\mathcal F_0
\right]
=
\operatorname{Var}(R\mid\mathcal F_0)
\leq \frac14 .
\]
Consequently, for any \(\epsilon>0\),
\[
\mathbb E
\left[
\sum_{i=1}^{\tau}
\mathbf 1\{|C_i|>\epsilon\}
\middle|
\mathcal F_0
\right]
\leq
\frac{1}{4\epsilon^2}.
\]
\end{restatable}

We now examine how this sparsity property relates to GAE
in the outcome-only reward setting, where the reward is
revealed only after the generation trajectory is complete.
Let the critic estimate be
\(
\widehat V_i=V_i+\varepsilon_i,
\)
where \(\varepsilon_i\) denotes the value estimation error. With \(\gamma=1\) and the terminal prediction set to
\(\widehat V_\tau=R\), the TD residual becomes
\(
\widehat\delta_i
=
\widehat V_i-\widehat V_{i-1}
=
C_i+\varepsilon_i-\varepsilon_{i-1}.
\)
For GAE with parameter \(\lambda\), the estimated advantage is computed as
\(
\widehat A_t^\lambda
=
\sum_{i=t}^{\tau}
\lambda^{i-t}\widehat\delta_i .
\)
Substituting the above TD decomposition gives
\(
\widehat A_t^\lambda
=
\sum_{i=t}^{\tau}
\lambda^{i-t}C_i
-\varepsilon_{t-1}
+
(1-\lambda)
\sum_{i=t}^{\tau-1}
\lambda^{i-t}\varepsilon_i ,
\)
where the terminal value error is assumed to be zero.

This decomposition separates the true credit signal from critic estimation
errors. The first term corresponds to the accumulated token-level credit,
while the remaining terms arise from imperfect value estimation. When
\(\lambda<1\), intermediate critic errors are retained through the last term.
However, Theorem~\ref{lem:credit_dilution} bounds the expected
number of credit increments exceeding any fixed magnitude,
independently of the maximum response length. Consequently, the critic
error term can become comparable to, or even dominate, the true credit signal.

In contrast, when \(\lambda=1\), the intermediate critic error term vanishes:
\(
\widehat A_t^1
=
\sum_{i=t}^{\tau}C_i-\varepsilon_{t-1}
=
R-\widehat V_{t-1}.
\)
Therefore, \(\lambda=1\) eliminates intermediate value
estimation errors, leaving only the prefix value error
\(-\varepsilon_{t-1}\).

Following the analysis above, our proposed method, i.e., PACT, uses $\lambda=1$.

\section{Policy Aligned Critic Training}
\label{sec:PACT}

\begin{figure}[t]
    \centering
    \includegraphics[width=1.0\linewidth]{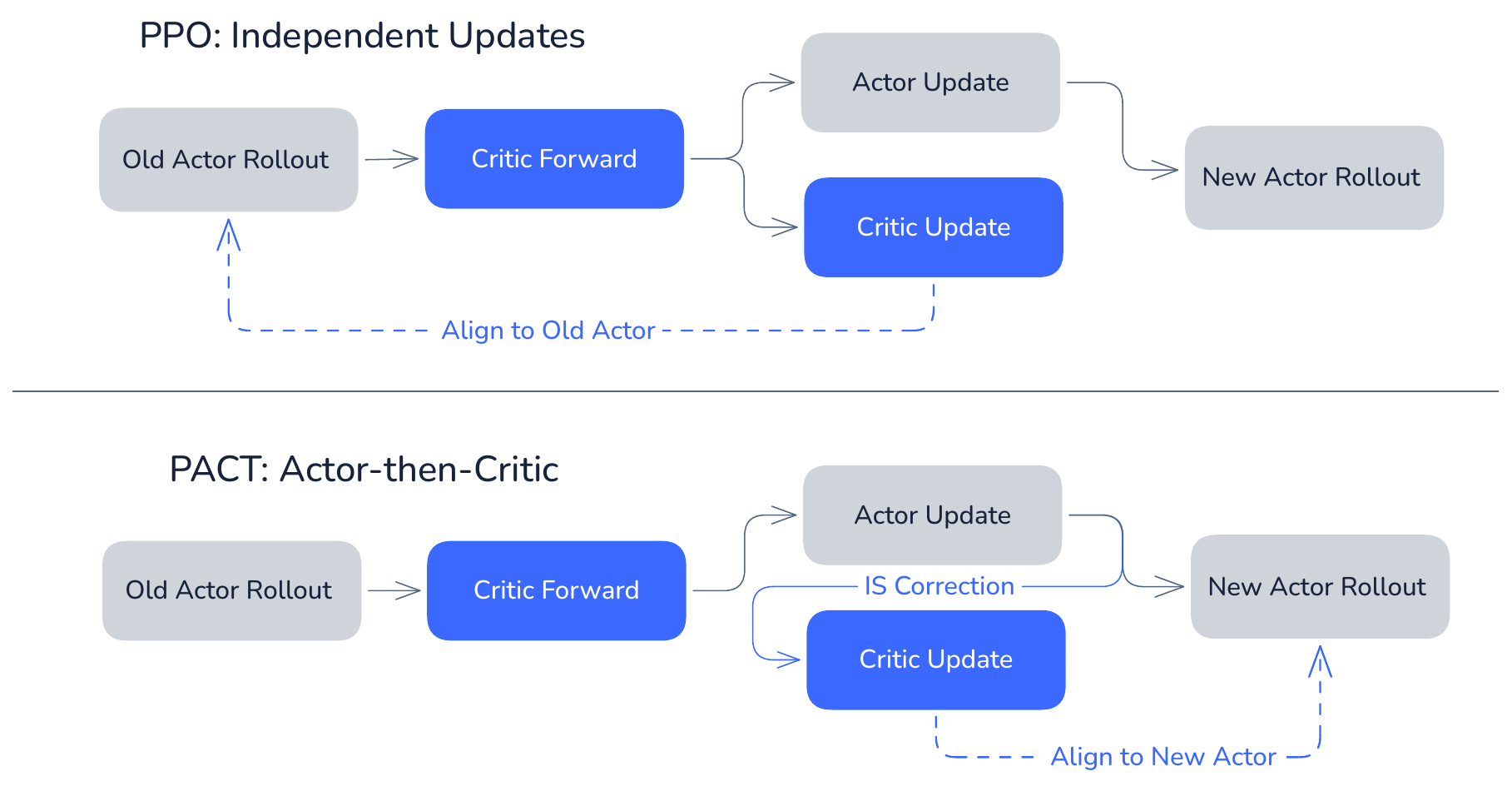}
    \caption{
    PPO and PACT training workflows.
    }
    \label{fig:PACT}
\end{figure}

\subsection{Motivation}
Theorems~\ref{thm:unique_credit} and~\ref{lem:credit_dilution}
highlight the importance of accurate value estimation for
fine-grained credit assignment. By Theorem~\ref{thm:unique_credit}, the unique
token-level credit under policy \(\pi\) is
\[
C_i^\pi
=
V_i^\pi-V_{i-1}^\pi,
\qquad
V_i^\pi
=
\mathbb E_\pi[R\mid\mathcal F_i].
\]
Thus, recovering local credit requires not only accurate value predictions,
but also values corresponding to the policy currently being optimized.
Moreover, Theorem~\ref{lem:credit_dilution} shows that most local credits can
be small in long-horizon generation. Consequently, even a modest value error
or policy mismatch may dominate the underlying credit signal.

Under current RL training frameworks, PPO-style training can
introduce a policy mismatch between the actor and the critic. Let \(\pi_k\) denote the policy used to
generate the trajectories \(\mathcal D_k\) in iteration \(k\). The critic
values used to compute advantages on \(\mathcal D_k\) are produced before the
critic is updated on \(\mathcal D_k\). Hence, these values are generated by
critic parameters learned from the previous trajectories
\(\mathcal D_{k-1}\sim\pi_{k-1}\):
\[
\mathcal D_k\sim\pi_k,
\qquad
\widehat V_{\phi_{k-1}}
\approx
V^{\pi_{k-1}}.
\]
Although the critic is subsequently updated on \(\mathcal D_k\), the resulting
critic is not used to recompute the advantages for the current actor update.
After the actor is updated from \(\pi_k\) to \(\pi_{k+1}\), the critic is
therefore again one policy update behind.

This policy--critic lag is especially consequential when token-level credit is
obtained by differencing consecutive conditional values. It motivates
training procedures that improve value estimation while explicitly
synchronizing the critic with the updated policy. We shall develop such a procedure
in this section.

\subsection{Probabilistic Value Estimation}

As discussed above and proved in
Appendix~\ref{app:reward_normalization}, any bounded reward can be normalized
to \([0,1]\) without changing the optimal policy. We therefore assume
\(R\in[0,1]\), which also implies
\(
V_i^\pi
=
\mathbb E_\pi[R\mid\mathcal F_i]
\in[0,1].
\)

Instead of the conventional mean squared error, we parameterize the critic as
\(
\widehat V_{\phi,i}
=
\sigma(z_{\phi,i})
\)
and optimize the soft binary cross-entropy loss
\[
\mathcal L_{\mathrm{BCE}}(\phi)
=
\mathbb E
\left[
-R\log\widehat V_{\phi,i}
-(1-R)\log(1-\widehat V_{\phi,i})
\right].
\]
For \(R\in[0,1]\), BCE and MSE have the same optimal prediction:
\[
\arg\min_{v\in[0,1]}
\mathbb E[
-R\log v-(1-R)\log(1-v)
\mid\mathcal F_i]
=
\arg\min_{v\in\mathbb R}
\mathbb E[(v-R)^2\mid\mathcal F_i]
=
V_i^\pi.
\]
At the endpoints, BCE is defined by its limits, allowing
the value \(+\infty\), with \(0\log 0=0\). Thus, using BCE does not change the value being estimated.
Classification-based objectives have shown favorable empirical performance
for value estimation in deep reinforcement learning
\cite{farebrother2024stop}. Furthermore, a controlled critic-pretraining ablation, reported in
Section~\ref{sec:opcp}, shows that the probabilistic BCE
critic converges substantially faster than the conventional MSE critic.

\subsection{Off-Policy Critic Synchronization}

As shown in Figure~\ref{fig:PACT}, Policy Aligned Critic Training (PACT) introduces an
\emph{Actor-then-Critic} dependency through importance-corrected
critic targets, whereas standard PPO allows independent actor
and critic updates. Within each
iteration, the current critic first provides the value predictions required
for actor optimization. After all actor updates have been completed, we
perform one additional forward pass over the same rollout batch using the
updated actor. Together with the log-probabilities recorded before the update,
this produces importance ratios between the pre-update and post-update
policies. We then train the critic using these ratios, so that the iteration
ends with the critic better aligned to the updated actor.

Let \(\mu=\pi_k\) be the policy before the actor update and
\(\pi=\pi_{k+1}\) the updated policy. For a prefix \(\mathcal F_{t-1}\), define the continuation
importance ratio
\[
I_t
=
\prod_{k=t}^{\tau}
\frac{\pi(T_k\mid\mathcal F_{k-1})}
{\mu(T_k\mid\mathcal F_{k-1})}.
\]
Assuming unchanged environment dynamics and absolute
continuity of the continuation distributions, a change
of measure gives
\[
V_{t-1}^\pi
=
\mathbb E_\pi[R\mid\mathcal F_{t-1}]
=
\mathbb E_\mu[I_tR\mid\mathcal F_{t-1}].
\]

We now regard \(I_tR\) as a single random variable. As established in the
previous subsection, binary cross-entropy recovers the conditional mean:
for any integrable random variable \(Y\) whose conditional mean lies in
\([0,1]\),
\[
\mathbb E[Y\mid\mathcal F]
=
\arg\min_{x\in[0,1]}
\mathbb E
\left[
\mathcal L_{\mathrm{BCE}}(x,Y)
\mid\mathcal F
\right].
\]
Throughout this subsection, the conditional expected BCE is extended
to \(x=0\) and \(x=1\) by taking one-sided limits after the conditional
expectation. For a sigmoid-parameterized critic, optima at \(0\) and
\(1\) are approached as the logit tends to \(-\infty\) and \(+\infty\),
respectively.

Applying this result with \(Y=I_tR\) immediately gives
\[
V_{t-1}^\pi
=
\arg\min_{x\in[0,1]}
\mathbb E_{\mu}
\left[
\mathcal L_{\mathrm{BCE}}(x,I_tR)
\mid\mathcal F_{t-1}
\right].
\]
Therefore, trajectories sampled before the actor update can be used to train
a critic for the updated actor by replacing the original reward target \(R\)
with the importance-corrected target \(I_tR\). Although an individual realization of \(I_tR\) need not lie in \([0,1]\),
this does not alter the conditional minimizer above; a logit-space gradient
analysis is provided in Appendix~\ref{app:importance_weighted_bce}.
In practice, the exact continuation ratio can have high variance for long
responses. We use the detached current-token importance ratio and mask out
token-level critic losses whose ratios fall outside
$[\rho_{\min},\rho_{\max}]$. These
ratios require only one additional forward pass after actor training and no
additional rollout generation, yielding a stable and inexpensive surrogate
for exact policy synchronization.

\section{Experiments}
\begin{table*}[t]
    \centering
    \caption{
        Mathematical reasoning performance on Avg@16
        accuracy (\%).
    }
    \label{tab:agentic_math}
    \small
    \setlength{\tabcolsep}{6pt}
    \begin{tabular*}{\linewidth}{@{\extracolsep{\fill}}lccccc@{}}
        \toprule
        Method
        & AIME 2025
        & AIME 2026
        & BeyondAIME
        & HMMT Nov.\ 2025
        & Average \\
        \midrule
        Base Model
        & 46.67 & 51.25 & 28.63 & 37.50 & 41.01 \\
        GRPO (\(\epsilon_{\mathrm{high}}=0.28\))
        & 76.50 & 74.78 & 41.81 & 63.19 & 64.07 \\
        PPO (\(\lambda=0.95\))
        & 32.33 & 29.38 & 17.86 & 25.67 & 26.31 \\
        PPO (\(\lambda=1.0\))
        & 66.04 & 73.96 & 40.31 & 58.54 & 59.71 \\
        SAO
        & 51.25 & 63.33 & 36.63 & 53.33 & 51.14 \\
        \midrule
        PACT w/o IS
        & 76.04 & 82.50 & 51.38 & 61.04 & 67.74 \\
        \textbf{PACT}
        & \textbf{83.12} & \textbf{85.21}
        & \textbf{51.69} & \textbf{71.46}
        & \textbf{72.87} \\
        \bottomrule
    \end{tabular*}
\end{table*}

\subsection{Experimental Setup}

All RL training experiments use Dressage~\citep{dressage_github},
an agentic RL framework built on slime~\citep{slime_github}
that integrates agent execution, trajectory collection,
and policy training.

\paragraph{Training Details.}
For mathematical reasoning, we train
Qwen3.5-4B~\citep{qwen3.5} on a subset of
DAPO-Math-17k~\citep{yu2026dapo} using the
OpenCode~\citep{opencode_github} agentic harness, with
final-answer correctness as the outcome reward.
The training subset contains \(3{,}200\) problems, constructed
by prioritizing problems that the initial policy fails under
pass@1 evaluation and randomly sampling additional problems
from the remaining pool.
For agentic coding, we train
Qwen3.6-35B-A3B~\citep{qwen36_35b_a3b}
on OpenSWE~\citep{fu2026davincienvopensweenvironment}
using a Codex~\citep{codex_github} agent through
Harbor~\citep{Harbor_Framework}, with terminal rewards
provided by the task verifier.
For both tasks, each rollout round produces \(512\)
trajectories. GRPO samples \(8\) trajectories for each of
\(64\) prompts. The optimization batch size is \(128\),
yielding four minibatches per rollout round. SAO uses DIS with importance ratio ranges of \([0.7,6.0]\)
for mathematical reasoning and \([0.6,3.0]\) for coding.
PACT uses the same actor-side DIS range as SAO for mathematical
reasoning and PPO clipping for coding.
For critic training, PACT masks out samples whose importance
ratios fall outside \([0,6]\). For both tasks, the context window is \(128\mathrm{k}\) tokens,
and the maximum generation length per interaction turn is
\(64\mathrm{k}\) tokens.

\paragraph{Evaluation.}
For mathematical reasoning, we evaluate on AIME 2025,
AIME 2026, HMMT Nov.\ 2025~\citep{balunović2026matharenaevaluatingllmsuncontaminated},
and BeyondAIME~\citep{bytedance_seed_2025_beyondaime},
reporting Avg@16 accuracy.
We compare PACT with GRPO using Clip-Higher,
PPO with \(\lambda\in\{0.95,1.0\}\), and SAO.
For agentic coding, we evaluate on
SWE-bench Verified~\citep{jimenez2024swebench}
and compare PACT with SAO, GRPO with Clip-Higher,
and PPO with \(\lambda=1.0\).
\subsection{Main Results}

\begin{table*}[t]
    \centering
    \caption{
        SWE-bench Verified pass rates (\%)
        with Qwen3.6-35B-A3B.
    }
    \label{tab:agentic_coding}
    \small
    \setlength{\tabcolsep}{8pt}
    \begin{tabular}{@{}lccccc@{}}
        \toprule
        Method
        & Base Model
        & GRPO
        & PPO (\(\lambda=1.0\))
        & SAO
        & \textbf{PACT} \\
        \midrule
        Pass@1 Rate
        & 60.8 & 65.4 & 65.0 & 63.6 & \textbf{67.4} \\
        \bottomrule
    \end{tabular}
\end{table*}

Tables~\ref{tab:agentic_math} and~\ref{tab:agentic_coding}
report the results on mathematical reasoning and agentic coding,
respectively. On mathematical reasoning, PACT achieves the
highest accuracy on all four benchmarks, averaging \(72.87\%\)
and outperforming GRPO, PPO with \(\lambda=1.0\), and SAO
by \(8.80\), \(13.16\), and \(21.73\) percentage points,
respectively. PPO with \(\lambda=1.0\) also outperforms
\(\lambda=0.95\), which undergoes policy collapse during
training, consistent with our analysis of intermediate critic
errors in GAE. On SWE-bench Verified, PACT achieves the highest
pass rate of \(67.4\%\), outperforming GRPO,
PPO with \(\lambda=1.0\), and SAO by \(2.0\), \(2.4\),
and \(3.8\) percentage points, respectively.

\subsection{Ablation Studies}

\begin{figure*}[t]
    \centering
    \includegraphics[width=1.0\textwidth]
        {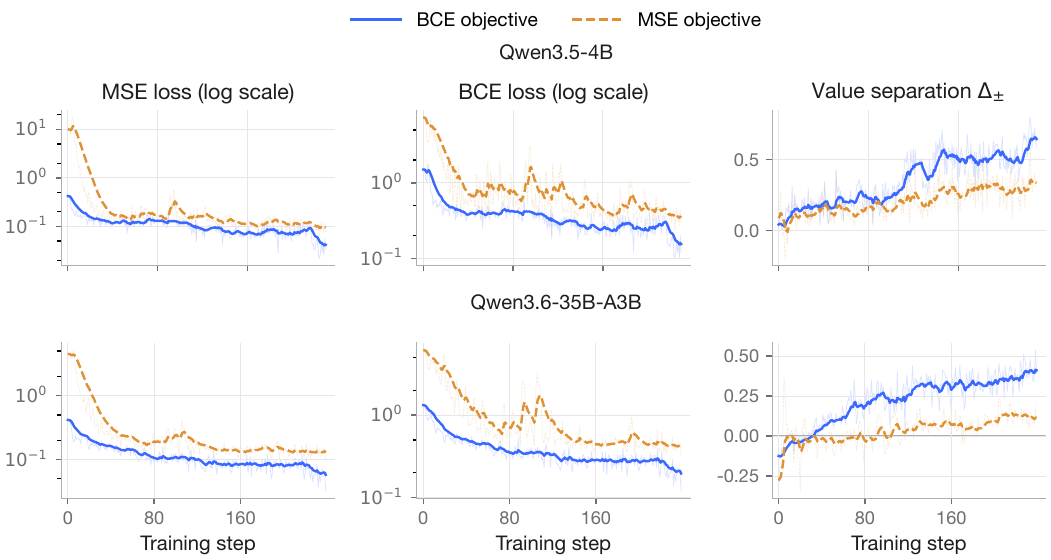}
    \caption{
        Fixed-policy critic pretraining with BCE and MSE
        objectives. Critics are trained from the same
        initialization on the same on-policy rollout data.
        We compare BCE loss, MSE loss, and value separation
        \(\Delta_{\pm}\).
    }
    \label{fig:critic_pretraining}
\end{figure*}

\paragraph{Critic Objective.}
\label{sec:opcp}

We isolate the effect of the critic objective by fixing the
rollout policy and training BCE and MSE critics from the same
initialization on the same on-policy rollout data.
We conduct this comparison in both settings described above:
Qwen3.5-4B on mathematical reasoning and Qwen3.6-35B-A3B
on agentic coding.

As shown in Figure~\ref{fig:critic_pretraining}, across both
model scales, BCE-trained critics achieve lower BCE and MSE
losses and greater value separation than their MSE-trained
counterparts.
Here, \(\Delta_{\pm}\) denotes the mean predicted value on
positive samples minus that on negative samples.
The larger separation indicates that BCE-trained critics
assign more distinct values to successful and unsuccessful
trajectories.

\paragraph{Importance Sampling Correction.}

PACT w/o IS retains the Actor-then-Critic update order,
BCE critic objective, and actor-side optimization settings,
but removes the critic-target importance sampling correction.
As shown in Table~\ref{tab:agentic_math}, adding importance
correction improves average accuracy from \(67.74\%\)
to \(72.87\%\), a gain of \(5.13\) percentage points,
with improvements on all four benchmarks.
It also leads to more stable training, as illustrated by
the training reward curves in Appendix~\ref{sec:training_dynamics}.
This comparison supports the effectiveness of importance
sampling correction in PACT.

\section{Conclusion}
We studied credit assignment under three regularity conditions,
namely Completeness, Prefix Consistency, and Neutrality,
and proved that credit is uniquely determined under these conditions.
The resulting representation provides a unified basis for
explaining phenomena across existing algorithms.
Under an ideal teacher, OPD yields an expected policy gradient
proportional to that induced by credit, identifying the teacher
as an implicit critic. RLOO yields the same expected policy
gradient contribution despite its coarser granularity.
We further established approximate credit sparsity under bounded
outcome rewards and analyzed how intermediate critic errors
affect GAE in long-horizon settings.

Motivated by these findings, we proposed Policy Aligned Critic
Training (PACT). Its Actor-then-Critic update order enables
importance sampling correction in critic training to better
align the critic with the updated policy. PACT also uses BCE
instead of MSE to improve value estimation.
Across four mathematical reasoning benchmarks, PACT achieves
\(72.87\%\) average accuracy, outperforming GRPO and PPO by
\(8.80\) and \(13.16\) percentage points, respectively.
On SWE-bench Verified, PACT achieves a pass rate of \(67.4\%\),
outperforming GRPO, PPO, and SAO by \(2.0\), \(2.4\), and
\(3.8\) percentage points, respectively.
Together, these results show how a mathematical characterization
of credit can explain existing algorithms and guide improvements
to actor-critic training.

\bibliographystyle{styles/colm2026_conference}
\bibliography{bibliography/references}

\clearpage
\appendix
\section{Contributors}

\begingroup
\renewcommand{\thefootnote}{%
  \ifcase\value{footnote}\or *\or \Letter\fi}

Jiayan Fu, Hang Xu\footnotemark[1], Yong Zhang,
Zhaokai Luo, Yao Hu, Dongyan Zhao\footnotemark[2],
Mu Chuan\footnotemark[2]

\footnotetext[1]{Project Leader.}
\footnotetext[2]{Corresponding Authors.
\nolinkurl{zhaody@pku.edu.cn},
\nolinkurl{muchuan1@xiaohongshu.com}.}
\endgroup

\section{Additional Related Work}
\label{app:RW}

\subsection{Reinforcement Learning for Large Language Models}

Reinforcement learning has become an important component of large language model post-training, supporting both preference alignment and the development of complex reasoning capabilities \citep{ouyang2022training,guo2025deepseek}. Its early applications focused primarily on preference alignment \citep{zhang2025survey}. Representative RLHF methods train scalar reward models from human preferences over model responses and use PPO to optimize language model policies \citep{stiennon2020learning,ouyang2022training,bai2022training}. Their PPO implementations use learned value functions to construct token-level advantage estimates \citep{stiennon2020learning,ouyang2022training}. More recently, reinforcement learning has increasingly been used to improve complex reasoning through outcome-level or verifiable rewards \citep{shao2024deepseekmath,guo2025deepseek,team2025kimi,zhang2025survey}.

To avoid the computational cost and optimization difficulty of training a critic, ReMax~\citep{li2023remax} uses the reward of a greedily decoded response as its baseline, whereas RLOO~\citep{ahmadian2024back} constructs a leave-one-out baseline from the rewards of other responses sampled for the same prompt. GRPO~\citep{shao2024deepseekmath} instead replaces the learned value model with group-relative advantage estimates. REINFORCE++~\citep{hu2025reinforce++} combines critic-free optimization with global advantage normalization. Subsequent work refines different components of critic-free optimization, including clipping, sampling, and loss aggregation in DAPO~\citep{yu2026dapo}, normalization-bias correction in Dr.\ GRPO~\citep{liu2025understanding}, and sequence-level importance weighting and clipping in GSPO~\citep{zheng2025groupsequencepolicyoptimization}.

Despite the prevalence of critic-free optimization, recent work has continued to develop value-based methods for long and heterogeneous reasoning trajectories. VC-PPO~\citep{yuan2025s} attributes observed failures of PPO in long-chain-of-thought training to value-initialization bias and the attenuation of terminal reward signals in GAE, and accordingly introduces value pretraining and decoupled GAE. VAPO~\citep{yue2025vapo} builds on these techniques and introduces length-adaptive GAE for responses of varying lengths. The high and variable cost of generating long trajectories has also motivated asynchronous training systems \citep{fu2026areal,hou2026single}. AReaL~\citep{fu2026areal} decouples rollout generation from policy optimization and explicitly accounts for stale training samples, whereas SAO~\citep{hou2026single} combines single-rollout asynchronous training with a learned value model and a token-level GAE estimator. In contrast to these studies of optimization algorithms and training systems, our work develops a general characterization of token-level credit in long-sequence LLM reinforcement learning and studies its implications for response-level baselines and learned critics.

\subsection{Credit Assignment in Reinforcement Learning}

Credit assignment concerns how earlier decisions contribute to
subsequent outcomes and has been recognized as a fundamental
difficulty in learning systems since early work on artificial
intelligence~\citep{minsky1961steps}.
Many classical reinforcement learning methods connect decisions
to outcomes through sampled returns and temporal structure.
Monte Carlo methods use sampled returns as targets for value
estimation~\citep{sutton1998reinforcement}, whereas
REINFORCE-style policy gradient methods use observed rewards
or returns to construct gradient
estimates~\citep{williams1992simple}.
Temporal-difference learning bootstraps from temporally
successive predictions~\citep{sutton1988learning}, while
eligibility traces distribute subsequent TD errors to previously
visited states and actions~\citep{sutton1998reinforcement}.
Actor-critic methods use learned value functions to estimate
policy gradients~\citep{sutton1999policy}.
Generalized advantage estimation combines TD residuals to
control the bias-variance trade-off in advantage
estimation~\citep{schulman2015high}.

When consequential decisions are separated from their outcomes
by long temporal gaps, subsequent work has developed more
targeted mechanisms.
RUDDER learns return-equivalent reward redistributions through
contribution analysis of return
predictions~\citep{arjona2019rudder}, while Temporal Value
Transport uses attentional memory retrieval to transport value
estimates from later events to relevant events in the distant
past~\citep{hung2019optimizing}.
Hindsight Credit Assignment assigns credit to past decisions
according to the likelihood that they led to an observed
outcome~\citep{harutyunyan2019hindsight}.
Counterfactual Credit Assignment constructs future-conditioned
baselines to disentangle an action's influence from external
factors and subsequent actions~\citep{mesnard2020counterfactual},
whereas COCOA estimates an action's contribution by asking
whether a subsequent reward would still have been obtained
under an alternative action~\citep{meulemans2023igottenrewardlongterm}.

\citet{pignatelli2023survey} organize these developments into
several methodological families, including approaches based
on temporal contiguity, return decomposition, and future
conditioning.
Collectively, these methods address the propagation of delayed
rewards, the decomposition of trajectory returns, or the
estimation of action influence, but they do so through different
target quantities and estimation procedures.
In long-horizon reinforcement learning for large language models,
this problem becomes especially salient, as many outcome-supervised
methods use a single response-level or terminal reward to train
a large number of token-level decisions.
Our work studies the mathematical characterization of credit
under three regularity conditions and uses the resulting unique
representation to analyze existing algorithms and guide
critic training.

\section{Detailed Proofs}
\subsection{Reward Normalization and Variance Bound}
\label{app:reward_normalization}

For a bounded reward \(R\) satisfying
\(R_{\min}\le R\le R_{\max}\) with \(R_{\min}<R_{\max}\),
define the normalized reward by
\[
\widetilde R
=
\frac{R-R_{\min}}{R_{\max}-R_{\min}}.
\]
Since this is a positive affine transformation, i.e.,
\[
\widetilde R=aR+b,\qquad a>0,
\]
maximizing the expected reward is equivalent to maximizing the expected
normalized reward.

It remains to show the variance bound for \(R\in[0,1]\). Let
\[
\mu=\mathbb E[R].
\]
Since \(0\leq R\leq1\), we have
\[
R^2\leq R.
\]
Therefore,
\[
\operatorname{Var}(R)
=
\mathbb E[R^2]-\mu^2
\leq
\mu-\mu^2
=
\mu(1-\mu)
\leq
\frac14.
\]

\subsection{Unique Representation Theorem}
\subsubsection{Prefix Consistency Implies Adaptedness}
\label{app:prefix_consistency}

\begin{lemma}
\label{lem:prefix_consistency_adaptedness}
Suppose that the prompt and all environmental observations are represented
as finite token sequences over a finite vocabulary. If a token-level credit
assignment satisfies Prefix Consistency, then its accumulated-credit process
\(\{S_i\}_{i=0}^{L}\) is adapted to the prefix filtration
\(\{\mathcal F_i\}_{i=0}^{L}\); that is, \(S_i\) is
\(\mathcal F_i\)-measurable for every \(i=0,\ldots,L\).
\end{lemma}

\begin{proof}
Let \((\Omega,\mathcal F,\mathbb P)\) denote the underlying probability space. We assume that all information presented to the language model, including the prompt \(q\) and each environmental observation \(O_j\), is represented as a finite token sequence over the finite vocabulary \(\mathcal V\). Let
\[
\mathcal V^*
=
\bigcup_{m=0}^{\infty}\mathcal V^m
\]denote the set of all finite token sequences. Since \(\mathcal V\) is finite, \(\mathcal V^*\) is countable.
For any fixed step \(i\), define the observed prefix
\[
H_i
=
(q,T_1,O_1,\ldots,T_i,O_i).
\]Then \(H_i\) takes values in the countable prefix space
\[
\mathcal H_i
=
\mathcal V^*
\times
(\mathcal V\times\mathcal V^*)^i,
\]which we equip with the discrete sigma-algebra \(2^{\mathcal H_i}\). By construction,
\[
\mathcal F_i=\sigma(H_i).
\]Prefix Consistency requires that, for any \(\omega,\widetilde\omega\in\Omega\),
\[
H_i(\omega)=H_i(\widetilde\omega)
\quad\Longrightarrow\quad
S_i(\omega)=S_i(\widetilde\omega).
\]Hence, for each \(h\in H_i(\Omega)\), the random variable \(S_i\) takes a common value on the fiber
\[
H_i^{-1}(\{h\}).
\]Define \(g_i:\mathcal H_i\to\mathbb R\) by assigning this common value to \(g_i(h)\) for \(h\in H_i(\Omega)\), and set \(g_i(h)=0\) for \(h\notin H_i(\Omega)\). Prefix Consistency ensures that \(g_i\) is well defined.
Because \(\mathcal H_i\) is equipped with the discrete sigma-algebra, \(g_i\) is measurable. Moreover, by construction,
\[
S_i=g_i\circ H_i.
\]Since \(H_i\) is \(\mathcal F_i\)-measurable, it follows that \(S_i\) is also \(\mathcal F_i\)-measurable. As this argument applies to every \(i\), the accumulated-credit process is adapted to the prefix filtration.

\end{proof}

\subsubsection{Proof of Unique Representation Theorem}

\uniquecredit*

\begin{proof}
Since $\tau\leq L$ almost surely, we adopt the convention that
$C_i=0$ on $\{i>\tau\}$ and regard every credit sequence as being defined
on the deterministic horizon $\{1,\ldots,L\}$. We similarly define
\[
V_i:=\mathbb E[R\mid\mathcal F_i],
\qquad i=0,\ldots,L.
\]
This extension is innocuous. Indeed, since $R$ is
$\mathcal F_\tau$-measurable, we have
\[
V_i=R
\qquad\text{on }\{\tau\leq i\}
\quad\text{a.s.}
\]
Consequently, $V_i-V_{i-1}=0$ on $\{i>\tau\}$ almost surely.

\paragraph{Existence.}
Define
\[
C_i^\star:=V_i-V_{i-1},
\qquad i=1,\ldots,L.
\]
We verify that this assignment satisfies the three regularity conditions.

First, the sum telescopes:
\[
\sum_{i=1}^{\tau}C_i^\star
=
V_\tau-V_0.
\]
Because $R$ is $\mathcal F_\tau$-measurable,
\[
V_\tau
=
\mathbb E[R\mid\mathcal F_\tau]
=
R
\quad\text{a.s.}
\]
Therefore,
\[
\sum_{i=1}^{\tau}C_i^\star
=
R-\mathbb E[R\mid\mathcal F_0]
\quad\text{a.s.},
\]
which establishes Completeness.

Next, let
\[
S_i^\star
:=
\sum_{j=1}^{i}C_j^\star
=
V_i-V_0.
\]
Recall that $\mathcal F_i=\sigma(H_i)$ and that the prefix space
$\mathcal H_i$ is countable. Hence, we may choose a version of the
conditional expectation for which
\[
V_i=v_i(H_i)
\]
for some measurable function $v_i:\mathcal H_i\to\mathbb R$. Likewise,
$V_0=v_0(q)$ for some measurable function $v_0$. It follows that
\[
S_i^\star
=
v_i(H_i)-v_0(q).
\]
Thus, if two trajectories have the same prefix up to step $i$, they have
the same value of $S_i^\star$. This proves Prefix Consistency.

Finally, the tower property of conditional expectation gives
\[
\begin{aligned}
\mathbb E[C_i^\star\mid\mathcal F_{i-1}]
&=
\mathbb E[V_i-V_{i-1}\mid\mathcal F_{i-1}]\\
&=
\mathbb E\!\left[
    \mathbb E[R\mid\mathcal F_i]
    \,\middle|\,\mathcal F_{i-1}
\right]
-
V_{i-1}\\
&=
\mathbb E[R\mid\mathcal F_{i-1}]
-
V_{i-1}\\
&=0
\quad\text{a.s.}
\end{aligned}
\]
Therefore, $C_i^\star$ also satisfies Neutrality, proving existence.

\paragraph{Uniqueness.}
Let $\{\widetilde C_i\}_{i=1}^{\tau}$ be any other credit assignment
satisfying Completeness, Prefix Consistency, and Neutrality. Extend it by
setting $\widetilde C_i=0$ on $\{i>\tau\}$, and define its accumulated
credit process by
\[
\widetilde S_i
:=
\sum_{j=1}^{i}\widetilde C_j,
\qquad i=0,\ldots,L,
\]
with $\widetilde S_0=0$.

By Prefix Consistency,
$\widetilde S_i$ depends only on the observed prefix $H_i$. Equivalently,
as established in Lemma~\ref{lem:prefix_consistency_adaptedness},
$\widetilde S_i$ is $\mathcal F_i$-measurable. Moreover, Neutrality yields
\[
\begin{aligned}
\mathbb E[\widetilde S_i\mid\mathcal F_{i-1}]
&=
\mathbb E[
    \widetilde S_{i-1}+\widetilde C_i
    \mid\mathcal F_{i-1}
]\\
&=
\widetilde S_{i-1}
+
\mathbb E[\widetilde C_i\mid\mathcal F_{i-1}]\\
&=
\widetilde S_{i-1}.
\end{aligned}
\]
Hence, $\{\widetilde S_i\}_{i=0}^{L}$ is a martingale with respect to
$\{\mathcal F_i\}_{i=0}^{L}$.

By Completeness and the zero extension after $\tau$, its terminal value is
\[
\widetilde S_L
=
\widetilde S_\tau
=
R-\mathbb E[R\mid\mathcal F_0]
=
R-V_0
\quad\text{a.s.}
\]
A finite-horizon martingale is determined by its terminal value. Therefore,
for every $i=0,\ldots,L$,
\[
\begin{aligned}
\widetilde S_i
&=
\mathbb E[\widetilde S_L\mid\mathcal F_i]\\
&=
\mathbb E[R-V_0\mid\mathcal F_i]\\
&=
\mathbb E[R\mid\mathcal F_i]-V_0\\
&=
V_i-V_0
\quad\text{a.s.}
\end{aligned}
\]
Taking successive differences gives
\[
\widetilde C_i
=
\widetilde S_i-\widetilde S_{i-1}
=
V_i-V_{i-1}
=
C_i^\star
\quad\text{a.s.}
\]
for every $i\leq\tau$. Thus, the credit assignment is unique up to
almost-sure equality.

The representation also immediately implies
\[
\mathbb E[C_i^\star\mid\mathcal F_{i-1}]=0,
\]
so $\{C_i^\star\}_{i=1}^{\tau}$ is a martingale difference sequence with
respect to $\{\mathcal F_i\}_{i=0}^{\tau}$.
\end{proof}

\subsubsection{Necessity of the Regularity Conditions}
\label{app:necessity_conditions}

The three regularity conditions in
Theorem~\ref{thm:unique_credit} are logically independent. The following
construction shows that removing any one of them admits credit assignments
that differ from the unique conditional-reward increments.

\begin{proposition}[Necessity of the three regularity conditions]
\label{prop:necessity_conditions}
None of Completeness, Prefix Consistency, and Neutrality is implied by the
other two. A single bounded two-step process suffices to show that none of the
three conditions is implied by the other two.
\end{proposition}

\begin{proof}
Let \(X\) and \(Y\) be independent Rademacher random variables:
\[
\mathbb P(X=1)
=
\mathbb P(X=-1)
=
\mathbb P(Y=1)
=
\mathbb P(Y=-1)
=
\frac12.
\]
Consider a stochastic process with a deterministic two-step horizon with
\[
\tau=2,
\qquad
\mathcal F_0=\{\varnothing,\Omega\},
\qquad
\mathcal F_1=\sigma(X),
\qquad
\mathcal F_2=\sigma(X,Y),
\]
and define the terminal reward by
\[
R=\frac{1+X}{2}.
\]
Thus \(R\in\{0,1\}\), and
\[
V_0=\mathbb E[R]=\frac12,
\qquad
V_1=\mathbb E[R\mid\mathcal F_1]=R,
\qquad
V_2=\mathbb E[R\mid\mathcal F_2]=R.
\]
The unique representation from
Theorem~\ref{thm:unique_credit} is therefore
\[
C_1^\star=\frac{X}{2},
\qquad
C_2^\star=0.
\]

\paragraph{Completeness}
Consider the assignment
\[
C_1=0,
\qquad
C_2=0.
\]
Its accumulated credits are
\[
S_1=0,
\qquad
S_2=0,
\]
so Prefix Consistency holds. Neutrality also holds because
\[
\mathbb E[C_1\mid\mathcal F_0]=0,
\qquad
\mathbb E[C_2\mid\mathcal F_1]=0.
\]
However,
\[
C_1+C_2
=
0
\neq
R-V_0
=
\frac{X}{2}.
\]
Thus an assignment may be prefix-consistent and neutral while explaining
none of the realized reward deviation.

\paragraph{Prefix Consistency}
Consider the assignment
\[
C_1=\frac{X}{2}+Y,
\qquad
C_2=-Y.
\]
It satisfies Completeness because
\[
C_1+C_2
=
\frac{X}{2}
=
R-V_0.
\]
It also satisfies Neutrality:
\[
\mathbb E[C_1\mid\mathcal F_0]
=
\mathbb E\left[\frac{X}{2}+Y\right]
=
0
\]
and, by the independence of \(X\) and \(Y\),
\[
\mathbb E[C_2\mid\mathcal F_1]
=
-\mathbb E[Y\mid X]
=
0.
\]
Nevertheless,
\[
S_1=C_1=\frac{X}{2}+Y
\]
is not \(\mathcal F_1\)-measurable. Indeed,
\[
\operatorname{Var}(S_1\mid\mathcal F_1)
=
\operatorname{Var}(Y\mid X)
=
1.
\]
Consequently, two trajectories with the same first-step information
\(X\) but different future realizations of \(Y\) assign different credit
to the first token. The assignment therefore uses future information to
revise the credit of an already generated prefix.

\paragraph{Neutrality}
Consider the assignment
\[
C_1=X,
\qquad
C_2=-\frac{X}{2}.
\]
It satisfies Completeness because
\[
C_1+C_2
=
\frac{X}{2}
=
R-V_0.
\]
Both \(S_1=X\) and \(S_2=X/2\) are measurable with respect to their corresponding prefix sigma-algebras. Thus, Prefix Consistency is satisfied.
However,
\[
\mathbb E[C_2\mid\mathcal F_1]
=
-\frac{X}{2}
\neq 0
\qquad\text{a.s.}
\]
The second-token credit is already completely predictable before the
second token is generated. Thus credit can be moved arbitrarily between
token positions through predictable compensating terms while preserving
both prefix measurability and the terminal sum.

The three constructions show that each condition excludes a distinct
pathology. Completeness prevents the assignment from ignoring part or
all of the realized outcome. Prefix Consistency prevents future
information from modifying credit assigned to an earlier prefix.
Neutrality prevents predictable zero-sum transfers of credit across
token positions. Hence all three conditions are necessary for the unique
representation in Theorem~\ref{thm:unique_credit}.
\end{proof}

\subsubsection{Aggregation of the Unique Token-Level Representation}
\Cor*
\begin{proof}
The result follows directly from Theorem~\ref{thm:unique_credit} by summing
the unique token-level credit increments within each segment.
\end{proof}

\subsection{Teacher Credit Equivalence Theorem}
\label{app:opd_credit}

\opd*

\begin{proof}
Condition on \(\mathcal F_{t-1}\). The optimization problem in
Definition~\ref{def:ideal_opd_teacher} can be written as
\[
\max_{q\in\Delta(\mathcal V)}
\left\{
\sum_{a\in\mathcal V}q(a)Q_t^\pi(a)
-
\beta
\sum_{a\in\mathcal V}
q(a)\log\frac{q(a)}{p_t(a)}
\right\}.
\]
Its Lagrangian is
\[
\mathcal L(q,\lambda)
=
\sum_{a\in\mathcal V}q(a)Q_t^\pi(a)
-
\beta
\sum_{a\in\mathcal V}
q(a)\log\frac{q(a)}{p_t(a)}
+
\lambda
\left(
\sum_{a\in\mathcal V}q(a)-1
\right).
\]
The first-order condition for each \(a\in\mathcal V\) is
\[
Q_t^\pi(a)
-
\beta
\left(
\log\frac{q(a)}{p_t(a)}+1
\right)
+
\lambda
=
0.
\]
Consequently, the unique optimizer is
\[
q_t^\star(a)
=
\frac{
p_t(a)\exp\!\left(Q_t^\pi(a)/\beta\right)
}{
\sum_{b\in\mathcal V}
p_t(b)\exp\!\left(Q_t^\pi(b)/\beta\right)
}.
\]
Define the prefix-dependent normalizing factor
\[
\mathcal Z_t
=
\sum_{b\in\mathcal V}
p_t(b)\exp\!\left(Q_t^\pi(b)/\beta\right).
\]
It follows that
\[
\log\frac{q_t^\star(a)}{p_t(a)}
=
\frac{1}{\beta}Q_t^\pi(a)
-
\log\mathcal Z_t.
\]
Substituting this identity into the OPD gradient yields
\[
\begin{aligned}
G_t^{\mathrm{OPD}}(q_t^\star)
&=
\frac{1}{\beta}
\mathbb E_\pi
\left[
Z_tQ_t^\pi(T_t)
\middle|
\mathcal F_{t-1}
\right]
-
\log\mathcal Z_t\,
\mathbb E_\pi
\left[
Z_t
\middle|
\mathcal F_{t-1}
\right].
\end{aligned}
\]
The conditional score-function identity gives
\[
\mathbb E_\pi
\left[
Z_t
\middle|
\mathcal F_{t-1}
\right]
=
0.
\]
Therefore,
\[
G_t^{\mathrm{OPD}}(q_t^\star)
=
\frac{1}{\beta}
\mathbb E_\pi
\left[
Z_tQ_t^\pi(T_t)
\middle|
\mathcal F_{t-1}
\right].
\]

It remains to relate the token-action value to the unique credit. Since
\[
C_t^\pi
=
V_t^\pi-V_{t-1}^\pi,
\]
the tower property gives
\[
\begin{aligned}
\mathbb E_\pi
\left[
C_t^\pi
\middle|
\mathcal F_{t-1},T_t
\right]
&=
\mathbb E_\pi
\left[
V_t^\pi
\middle|
\mathcal F_{t-1},T_t
\right]
-
V_{t-1}^\pi
\\
&=
Q_t^\pi(T_t)-V_{t-1}^\pi.
\end{aligned}
\]
Because \(Z_t\) is measurable with respect to
\(\sigma(\mathcal F_{t-1},T_t)\),
\[
\begin{aligned}
\mathbb E_\pi
\left[
Z_tC_t^\pi
\middle|
\mathcal F_{t-1}
\right]
&=
\mathbb E_\pi
\left[
Z_t
\mathbb E_\pi
\left[
C_t^\pi
\middle|
\mathcal F_{t-1},T_t
\right]
\middle|
\mathcal F_{t-1}
\right]
\\
&=
\mathbb E_\pi
\left[
Z_t
\left(
Q_t^\pi(T_t)-V_{t-1}^\pi
\right)
\middle|
\mathcal F_{t-1}
\right]
\\
&=
\mathbb E_\pi
\left[
Z_tQ_t^\pi(T_t)
\middle|
\mathcal F_{t-1}
\right].
\end{aligned}
\]
The last equality again follows from the conditional score-function
identity. Combining the two expressions proves
\[
G_t^{\mathrm{OPD}}(q_t^\star)
=
\frac{1}{\beta}
\mathbb E_\pi
\left[
Z_tC_t^\pi
\middle|
\mathcal F_{t-1}
\right].
\]
\end{proof}

\subsection{Gradient Unbiasedness of the RLOO Estimator}
Throughout the proof, all expectations are conditioned on the prompt \(q\).
Following the zero-extension convention, set both \(C_t\) and \(Z_t\)
to zero after termination.
Assume that all expectations involving policy scores below are finite.
\begin{proof}
Let
\[
Z_t=\nabla_\theta\log\pi_\theta(T_t\mid\mathcal F_{t-1}).
\]
By Theorem~\ref{thm:unique_credit}, the unique token-level credit satisfies
\(C_t=V_t-V_{t-1}\) and
\(
R_i-V_0=\sum_{k=1}^{\tau}C_k .
\)
We first show that
\[
\mathbb E[Z_t(R_i-C_t)]=0.
\]

For \(k<t\), \(C_k\) is \(\mathcal F_{t-1}\)-measurable, and therefore
\[
\mathbb E[Z_tC_k]
=
\mathbb E[C_k\mathbb E[Z_t|\mathcal F_{t-1}]]
=0.
\]
For \(k>t\), \(Z_t\) is \(\mathcal F_{k-1}\)-measurable, and the martingale
difference property of \(C_k\) gives
\[
\mathbb E[Z_tC_k]
=
\mathbb E[Z_t\mathbb E[C_k|\mathcal F_{k-1}]]
=0.
\]
The same argument applies to \(V_0\), since
\(V_0\) is \(\mathcal F_{t-1}\)-measurable. Hence
\[
\mathbb E[Z_t(R_i-C_t)]=0.
\]
For the RLOO baseline, \(\bar R_{-i}\) is computed from independent responses
and is independent of the current trajectory. 

Thus
\(
\mathbb E[Z_t\bar R_{-i}]=0.
\)
Combining the above results gives
\[
\mathbb E[Z_t(R_i-\bar R_{-i})]
=
\mathbb E[Z_tC_t].
\]
\end{proof}

\subsection{Approximate Credit Sparsity Theorem}

\ACS*

\begin{proof}
Since \(\tau\le L\) almost surely, extend the credit sequence to the
deterministic horizon \(\{1,\ldots,L\}\) by setting \(C_i=0\) on
\(\{i>\tau\}\). By the unique representation,
\(C_i=V_i-V_{i-1}\), where \(V_i=\mathbb E[R\mid\mathcal F_i]\).
Because \(R\in[0,1]\), these increments are square-integrable.

For \(1\le i<j\le L\), \(C_i\) is
\(\mathcal F_{j-1}\)-measurable. The tower property and the
martingale difference property therefore give
\[
\mathbb E[C_iC_j\mid\mathcal F_0]
=
\mathbb E\!\left[
C_i\,\mathbb E[C_j\mid\mathcal F_{j-1}]
\middle|\mathcal F_0
\right]
=0.
\]
By Completeness and the zero extension,
\[
\sum_{i=1}^{L}C_i
=
\sum_{i=1}^{\tau}C_i
=
R-V_0,
\qquad
V_0=\mathbb E[R\mid\mathcal F_0].
\]
Consequently,
\[
\begin{aligned}
\mathbb E\!\left[
\sum_{i=1}^{\tau}C_i^2
\middle|\mathcal F_0
\right]
&=
\sum_{i=1}^{L}\mathbb E[C_i^2\mid\mathcal F_0]\\
&=
\mathbb E\!\left[
\left(\sum_{i=1}^{L}C_i\right)^2
\middle|\mathcal F_0
\right]\\
&=
\mathbb E[(R-V_0)^2\mid\mathcal F_0]\\
&=
\operatorname{Var}(R\mid\mathcal F_0).
\end{aligned}
\]
Since \(R^2\le R\) and \(V_0\in[0,1]\),
\[
\operatorname{Var}(R\mid\mathcal F_0)
=
\mathbb E[R^2\mid\mathcal F_0]-V_0^2
\le V_0(1-V_0)
\le \frac14.
\]

Finally, for any \(\epsilon>0\), the pointwise inequality
\[
\epsilon^2
\sum_{i=1}^{\tau}\mathbf 1\{|C_i|>\epsilon\}
\le
\sum_{i=1}^{\tau}C_i^2
\]
implies, upon taking conditional expectations,
\[
\mathbb E\!\left[
\sum_{i=1}^{\tau}\mathbf 1\{|C_i|>\epsilon\}
\middle|\mathcal F_0
\right]
\le
\frac{\operatorname{Var}(R\mid\mathcal F_0)}{\epsilon^2}
\le
\frac{1}{4\epsilon^2}.
\]
\end{proof}
\subsection{Gradient of BCE with Importance-Weighted Targets}
\label{app:importance_weighted_bce}

For the prefix \(\mathcal F_{t-1}\), let the critic prediction be
\[
\widehat V_{\phi,t-1}=\sigma(z_{\phi,t-1}),
\]
where \(z_{\phi,t-1}\) is the critic logit. Let \(I_t\) be the
exact continuation importance ratio satisfying
\[
\mathbb E_\mu[I_tR\mid\mathcal F_{t-1}]
=
V_{t-1}^\pi.
\]
The importance-weighted target \(I_tR\) is held fixed when
differentiating with respect to the critic logit. Its BCE is
\[
\begin{aligned}
\operatorname{BCE}
\left(\sigma(z_{\phi,t-1}),I_tR\right)
&=
-I_tR\log\sigma(z_{\phi,t-1}) \\
&\quad-
(1-I_tR)\log\left(1-\sigma(z_{\phi,t-1})\right) \\
&=
\log\left(1+\exp(z_{\phi,t-1})\right)
-I_tRz_{\phi,t-1}.
\end{aligned}
\]
Therefore,
\[
\frac{\partial}{\partial z_{\phi,t-1}}
\operatorname{BCE}
\left(\sigma(z_{\phi,t-1}),I_tR\right)
=
\widehat V_{\phi,t-1}-I_tR.
\]
Taking the conditional expectation under the rollout policy gives
\[
\begin{aligned}
\mathbb E_\mu
\left[
\left.
\frac{\partial}{\partial z_{\phi,t-1}}
\operatorname{BCE}
\left(\sigma(z_{\phi,t-1}),I_tR\right)
\right|
\mathcal F_{t-1}
\right]
&=
\widehat V_{\phi,t-1}
-
\mathbb E_\mu[I_tR\mid\mathcal F_{t-1}] \\
&=
\widehat V_{\phi,t-1}-V_{t-1}^\pi.
\end{aligned}
\]
For \(0<V_{t-1}^\pi<1\), the expected logit gradient vanishes
exactly when
\[
\widehat V_{\phi,t-1}=V_{t-1}^\pi.
\]
When \(V_{t-1}^\pi=0\) or \(1\), equality is approached as
the logit tends to \(-\infty\) or \(+\infty\), respectively.
In addition,
\[
\frac{\partial^2}{\partial z_{\phi,t-1}^2}
\operatorname{BCE}
\left(\sigma(z_{\phi,t-1}),I_tR\right)
=
\widehat V_{\phi,t-1}
\left(1-\widehat V_{\phi,t-1}\right)
\leq \frac14.
\]
Thus, although individual importance-weighted targets may
fall outside \([0,1]\), the expected logit gradient has the
sign of the value prediction error, and the curvature with
respect to the logit remains uniformly bounded.

We further show that approaching the minimum expected BCE
implies approaching the target conditional mean in mean square.
Let \(Y\) be an integrable target and write
\[
m=\mathbb E[Y\mid\mathcal F]\in[0,1].
\]
For an \(\mathcal F\)-measurable prediction \(v\in(0,1)\), define
\[
\ell(v)
=
\mathbb E[\operatorname{BCE}(v,Y)\mid\mathcal F]
=
-m\log v-(1-m)\log(1-v).
\]
All logarithms are natural, and the expectations below are
assumed to exist and be finite. Differentiating gives
\[
\ell'(v)=\frac{v-m}{v(1-v)}.
\]
Since \(u(1-u)\leq 1/4\), for \(v\geq m\),
\[
\ell(v)-\ell(m)
=
\int_m^v\frac{u-m}{u(1-u)}\,du
\geq
4\int_m^v(u-m)\,du
=
2(v-m)^2.
\]
Similarly, for \(v<m\),
\[
\ell(v)-\ell(m)
=
\int_v^m\frac{m-u}{u(1-u)}\,du
\geq
4\int_v^m(m-u)\,du
=
2(v-m)^2.
\]
Here,
\[
\ell(m)=-m\log m-(1-m)\log(1-m),
\]
with \(0\log 0=0\). When \(m=0\) or \(m=1\), this expression
denotes the infimum over \(v\in(0,1)\), and the inequalities
follow by taking limits. Taking expectations gives
\begin{equation}
\mathbb E[(v-m)^2]
\leq
\frac12\,\mathbb E[\ell(v)-\ell(m)].
\label{eq:bce_value_error}
\end{equation}
Consequently, for any sequence of predictions \(v_n\),
\[
\mathbb E[\ell(v_n)-\ell(m)]\longrightarrow 0
\quad\Longrightarrow\quad
\mathbb E[(v_n-m)^2]\longrightarrow 0.
\]
This applies both to the on-policy target \(Y=R\) and to
the exact importance-weighted target \(Y=IR\), provided that
\(\mathbb E_\mu[IR\mid\mathcal F]=V^\pi\).

We can further relate value estimation error to credit
estimation error. Fix a policy \(\pi\) and a trajectory
distribution under which all the following expectations
are taken. For \(i\in\{t-1,t\}\), suppose that
\[
\mathbb E[Y_i\mid\mathcal F_i]=V_i^\pi,
\]
and let \(\ell_i\) denote the corresponding conditional
expected BCE. Define
\[
\varepsilon_i=\widehat V_{\phi,i}-V_i^\pi,
\qquad
\widehat C_{\phi,t}
=
\widehat V_{\phi,t}-\widehat V_{\phi,t-1}.
\]
Since \(C_t^\pi=V_t^\pi-V_{t-1}^\pi\),
\[
\widehat C_{\phi,t}-C_t^\pi
=
\varepsilon_t-\varepsilon_{t-1}.
\]
Using \((a-b)^2\leq 2a^2+2b^2\) and
Equation~\ref{eq:bce_value_error}, we obtain
\begin{equation}
\begin{aligned}
\mathbb E[(\widehat C_{\phi,t}-C_t^\pi)^2]
&\leq
2\mathbb E[\varepsilon_t^2]
+
2\mathbb E[\varepsilon_{t-1}^2] \\
&\leq
\mathbb E[
\ell_t(\widehat V_{\phi,t})-\ell_t(V_t^\pi)
] \\
&\quad+
\mathbb E[
\ell_{t-1}(\widehat V_{\phi,t-1})
-\ell_{t-1}(V_{t-1}^\pi)
].
\end{aligned}
\label{eq:bce_credit_error}
\end{equation}
Thus, if the expected BCE at both adjacent prefixes approaches
its minimum, the estimated credit converges to \(C_t^\pi\)
in mean square.

\subsection{Current-Token Importance Weighting as a One-Step Policy-Tracking Surrogate}
\label{app:one_step_policy_tracking}

Exact synchronization with an updated policy requires correcting the entire
continuation distribution. Such a correction involves a product of
token-level importance ratios and can be unstable for long trajectories. We
characterize here the current-token importance target used in our method and
its relation to the exact updated-policy value.

Let \(\mu\) denote the rollout policy and \(\pi\) the policy obtained after
actor optimization. For the token generated after prefix
\(\mathcal F_{t-1}\), define
\[
\rho_t
=
\frac{
\pi(T_t\mid\mathcal F_{t-1})
}{
\mu(T_t\mid\mathcal F_{t-1})
}.
\]
We assume that \(\pi(\cdot\mid\mathcal F_{t-1})\) is absolutely continuous
with respect to \(\mu(\cdot\mid\mathcal F_{t-1})\). The exact continuation
importance ratio is
\[
W_t
=
\prod_{k=t}^{\tau}\rho_k,
\]
and change of measure gives
\[
V_{t-1}^{\pi}
=
\mathbb E_\pi[R\mid\mathcal F_{t-1}]
=
\mathbb E_\mu[W_tR\mid\mathcal F_{t-1}].
\]
Our practical critic target instead uses only the current-token ratio,
\[
Y_t^{\mathrm{one}}
=
\rho_tR,
\qquad
\widetilde V_{t-1}^{\pi,\mu}
=
\mathbb E_\mu[\rho_tR\mid\mathcal F_{t-1}].
\]

\begin{proposition}[One-step policy tracking]
\label{prop:one_step_policy_tracking}
The surrogate value \(\widetilde V_{t-1}^{\pi,\mu}\) is exactly the value of
the hybrid policy that samples \(T_t\) from \(\pi\) and follows \(\mu\)
thereafter:
\[
\widetilde V_{t-1}^{\pi,\mu}
=
\mathbb E_{\pi\triangleright_t\mu}
[R\mid\mathcal F_{t-1}].
\]
Moreover, suppose that \(R\in[0,1]\), \(\tau\leq L\), and
\[
\left|\rho_k-1\right|\leq\varepsilon,
\qquad k=t,\ldots,\tau.
\]
Writing \(\delta_k=\rho_k-1\), we have
\[
V_{t-1}^{\pi}
=
\widetilde V_{t-1}^{\pi,\mu}
+
\mathbb E_\mu
\left[
R\sum_{k=t+1}^{\tau}\delta_k
\middle|
\mathcal F_{t-1}
\right]
+
\mathcal R_t,
\]
where
\[
|\mathcal R_t|
\leq
(1+\varepsilon)^{L-t+1}
-
1
-
(L-t+1)\varepsilon
=
O_L(\varepsilon^2).
\]
Consequently,
\[
\widetilde V_{t-1}^{\pi,\mu}
=
V_{t-1}^{\pi}
+
O_L(\varepsilon)
\]
as \(\varepsilon\to0\) for fixed \(L\).
\end{proposition}

\begin{proof}
Conditioning first on \(T_t\), we obtain
\[
\begin{aligned}
\mathbb E_\mu[\rho_tR\mid\mathcal F_{t-1}]
&=
\sum_a
\mu(a\mid\mathcal F_{t-1})
\frac{\pi(a\mid\mathcal F_{t-1})}
{\mu(a\mid\mathcal F_{t-1})}
\mathbb E_\mu
[R\mid\mathcal F_{t-1},T_t=a] \\
&=
\sum_a
\pi(a\mid\mathcal F_{t-1})
\mathbb E_\mu
[R\mid\mathcal F_{t-1},T_t=a].
\end{aligned}
\]
The last expression is precisely the expected reward obtained by sampling the
current token from \(\pi\) and using \(\mu\) for the remaining continuation.
This proves the hybrid-policy identity.

For the local expansion, observe that
\[
W_t
=
\prod_{k=t}^{\tau}(1+\delta_k)
=
1+\sum_{k=t}^{\tau}\delta_k+r_t,
\]
where \(r_t\) contains all products involving at least two distinct
\(\delta_k\)'s. Since \(|\delta_k|\leq\varepsilon\) and
\(\tau-t+1\leq L-t+1\),
\[
|r_t|
\leq
\sum_{m=2}^{L-t+1}
\binom{L-t+1}{m}\varepsilon^m
=
(1+\varepsilon)^{L-t+1}
-
1
-
(L-t+1)\varepsilon.
\]
Multiplying by \(R\), taking the conditional expectation, and using
\(0\leq R\leq1\) gives
\[
\begin{aligned}
V_{t-1}^{\pi}
&=
\mathbb E_\mu[W_tR\mid\mathcal F_{t-1}] \\
&=
\mathbb E_\mu[(1+\delta_t)R\mid\mathcal F_{t-1}]
+
\mathbb E_\mu
\left[
R\sum_{k=t+1}^{\tau}\delta_k
\middle|
\mathcal F_{t-1}
\right]
+
\mathcal R_t \\
&=
\widetilde V_{t-1}^{\pi,\mu}
+
\mathbb E_\mu
\left[
R\sum_{k=t+1}^{\tau}\delta_k
\middle|
\mathcal F_{t-1}
\right]
+
\mathcal R_t.
\end{aligned}
\]
The stated bounds follow immediately.
\end{proof}

For comparison, the complete first-order expansion of the continuation
importance target is
\[
Y_t^{\mathrm{FO}}
=
R\left(
1+\sum_{k=t}^{\tau}(\rho_k-1)
\right),
\]
which satisfies
\[
\mathbb E_\mu
[Y_t^{\mathrm{FO}}\mid\mathcal F_{t-1}]
=
V_{t-1}^{\pi}
+
O_L(\varepsilon^2).
\]
The practical target \(\rho_tR\) retains the current-token component of this
first-order correction and omits the remaining continuation terms. It is
therefore a one-step truncation of the full policy correction. This choice
avoids the multiplicative continuation weight while moving the critic target
toward the updated policy.

\section{Policy Aligned Critic Training Procedure}

Algorithm~\ref{alg:pact} summarizes PACT.
Here,  $\operatorname{ActorUpdate}$ denotes
its policy optimization step. All importance ratios and critic targets are
computed token-wise.

Notation: actor policy $\pi_{\theta_0}$;
probabilistic critic
$\widehat V_{\phi_0}(\mathcal F)=\sigma(z_{\phi_0}(\mathcal F))$;
number of iterations $K$;
actor/critic steps $B$;
importance-ratio acceptance bounds
$\rho_{\min},\rho_{\max}$

\begin{algorithm}[H]
\caption{Policy Aligned Critic Training}
\label{alg:pact}
\begin{algorithmic}
\Require $\pi_{\theta_0}$, $\widehat V_{\phi_0}$,
         $K$, $B$, $\rho_{\min}$, $\rho_{\max}$

\For{$k=0,\ldots,K-1$}

    \State $\mathcal D_k,R,\ell^{\mathrm{old}}
        \gets \Call{Rollout}{\pi_{\theta_k}}$

    \State $\widehat V
        \gets
        \operatorname{sg}
        \bigl(\widehat V_{\phi_k}(\mathcal D_k)\bigr)$

    \State $\widehat C
        \gets \Call{Credit}{R,\widehat V}$

    \State $\theta\gets\theta_k$

    \For{$b=1,\ldots,B$}
        \State $\theta
            \gets
            \Call{ActorUpdate}
            {\theta,\mathcal D_k,\widehat C,\ell^{\mathrm{old}}}$
    \EndFor

    \State $\theta_{k+1}\gets\theta$

    \State $\ell^{\mathrm{new}}
        \gets
        \Call{LogProb}{\pi_{\theta_{k+1}},\mathcal D_k}$

    \State $\rho
        \gets
        \operatorname{sg}\!\left[
        \exp\bigl(\ell^{\mathrm{new}}-\ell^{\mathrm{old}}\bigr)
        \right]$

    \State $\mathsf m
        \gets
        \mathbb I
        \bigl\{
        \rho_{\min}\leq\rho\leq\rho_{\max}
        \bigr\}$

    \State $Y\gets\rho\odot R$

    \State $\phi\gets\phi_k$

    \For{$b=1,\ldots,B$}

        \State $\displaystyle
        \mathcal L_{\mathrm{PACT}}(\phi)
        \gets
        \frac{
        \sum
        \mathsf m\odot
        \operatorname{BCE}
        \bigl(
        \widehat V_{\phi}(\mathcal D_k),
        Y
        \bigr)
        }{
        \sum\mathsf m
        }$

        \State $\phi
            \gets
            \phi-\eta_{\phi}
            \nabla_{\phi}
            \mathcal L_{\mathrm{PACT}}(\phi)$

    \EndFor

    \State $\phi_{k+1}\gets\phi$

\EndFor

\State \Return $\pi_{\theta_K},\widehat V_{\phi_K}$
\end{algorithmic}
\end{algorithm}

\section{Experimental Details}
\subsection{Prompt of Agentic Reasoning for Math}

After constructing the training subset, we convert each selected DAPO-Math-17k example and each evaluation example from a direct-answer problem into an agentic tool-use episode. We preserve the original problem statement and ground-truth final answer, while adding neither reference solutions nor tool-use demonstrations. The ground-truth answer is stored separately and is never included in the policy input.

Specifically, we prepend an instruction that informs the model of the available Python sandbox and explicitly encourages tool use for calculation, case enumeration, and verification. We also append a formatting instruction requiring the final answer to appear in a \verb|\boxed{}| expression on the last line. The resulting prompt template is:

\begin{quote}\small Solve the following math problem. You are in a sandbox with a working Python interpreter---USE it: write and run Python code to do the calculations, enumerate cases, and verify your result, instead of computing by hand.

\medskip\texttt{{original problem statement}}

\medskip Please put your final answer in \verb|\boxed{}|. The last line of your response must have the form \verb|Answer: \boxed{answer}|.\end{quote}

\begin{figure*}[t]
    \centering
    \includegraphics[width=0.96\textwidth]
    {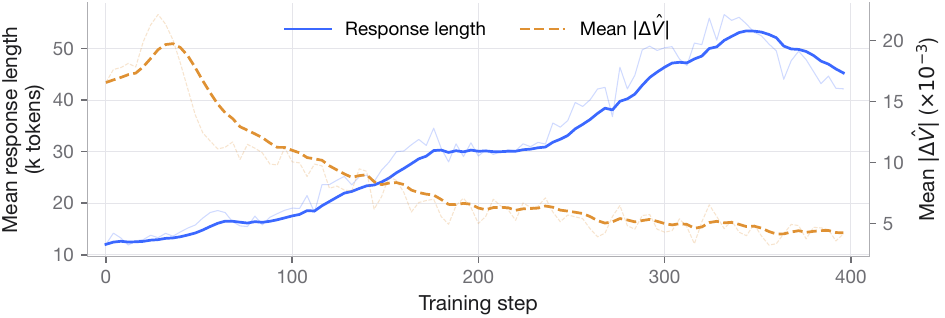}
    \caption{
        Dynamics of response length and mean
    \( |\widehat V_t-\widehat V_{t-1}| \) during PACT training.
    }
    \label{fig:credit_sparsity_dynamics}
\end{figure*}

\subsection{Further Empirical Evidence for Theorem~\ref{lem:credit_dilution}}
\label{app:further_empirical_evidence}

As shown in Figure~\ref{fig:credit_sparsity_dynamics}, responses become
progressively longer during PACT training, while the mean
\( |\widehat V_t-\widehat V_{t-1}| \) decreases. This trend is consistent with
the intuition behind Theorem~\ref{lem:credit_dilution}.

\subsection{Training Dynamics}
\label{sec:training_dynamics}

\begin{figure}[t]
    \centering
    \includegraphics[width=0.95\linewidth]
    {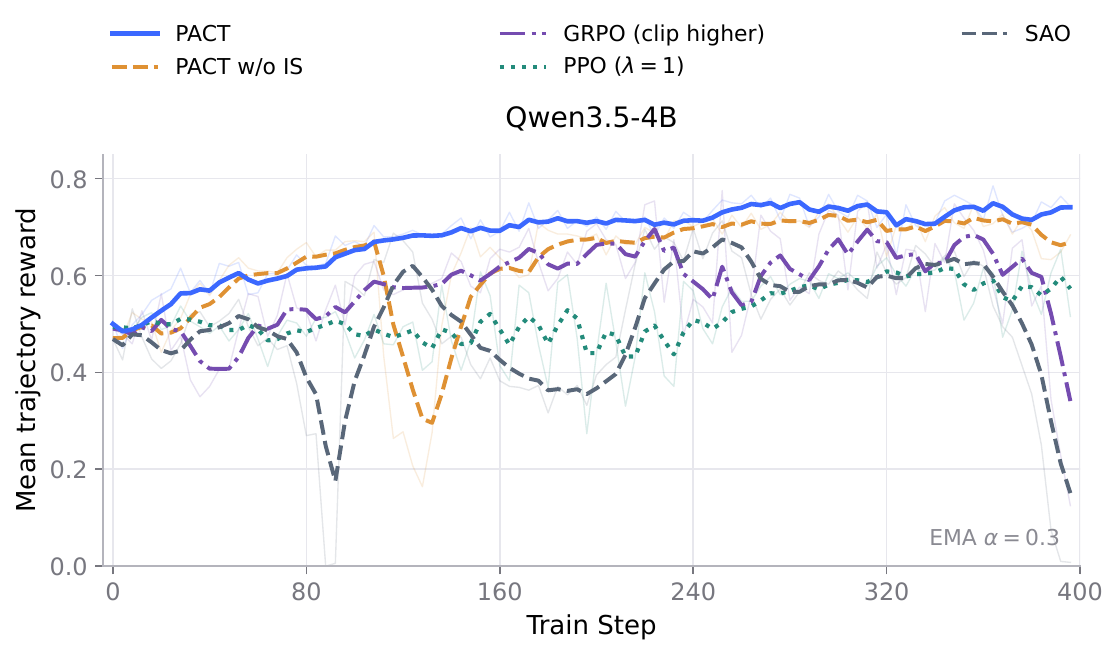}
    \caption{
        Training reward on agentic mathematical reasoning with
        Qwen3.5-4B. Faint lines show raw mean trajectory rewards,
        while bold lines show exponential moving averages
        with $\alpha=0.3$.
        Each rollout round corresponds to four training steps.
    }
    \label{fig:math_training_rewards}
\end{figure}

\begin{figure}[t]
    \centering
    \includegraphics[width=0.95\linewidth]
    {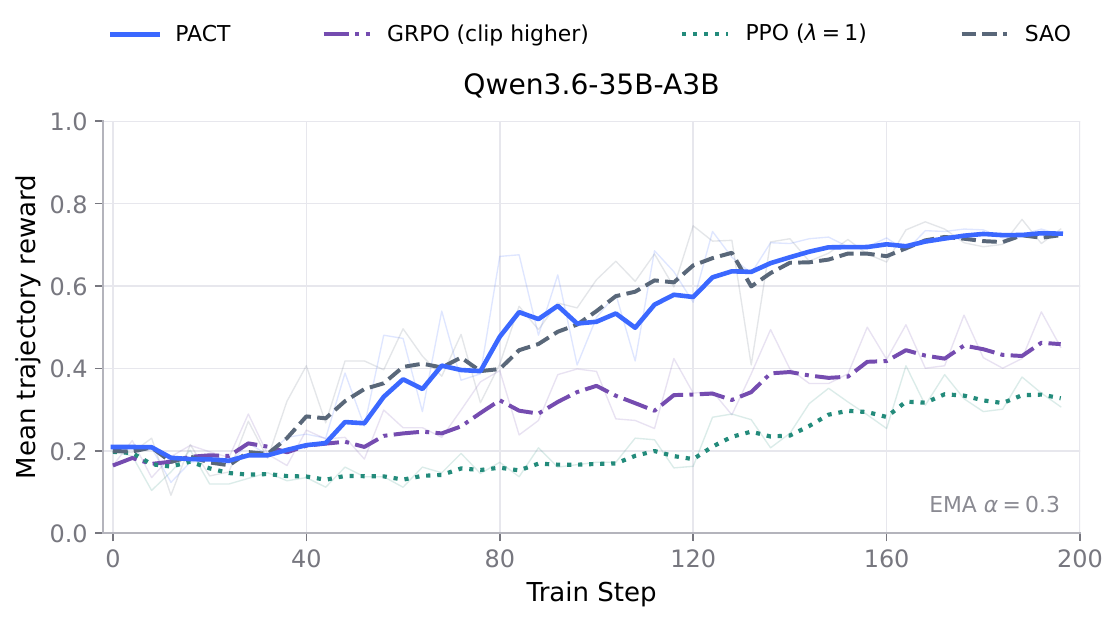}
    \caption{
        Training reward on agentic coding with Qwen3.6-35B-A3B.
        Faint lines show raw mean trajectory rewards,
        while bold lines show exponential moving averages
        with $\alpha=0.3$.
    }
    \label{fig:coding_training_rewards}
\end{figure}

Figures~\ref{fig:math_training_rewards}
and~\ref{fig:coding_training_rewards} show the evolution
of mean trajectory rewards during training on mathematical
reasoning and agentic coding, respectively.

\subsection{Key Training Hyperparameters}
\label{app:training_hyperparameters}

Table~\ref{tab:training_hyperparameters} summarizes the key
training hyperparameters for mathematical reasoning and coding.

\begin{table*}[t]
    \centering
    \caption{Key training hyperparameters.}
    \label{tab:training_hyperparameters}
    \small
    \setlength{\tabcolsep}{8pt}
    \begin{tabular}{lcc}
        \toprule
        Hyperparameter & Mathematical Reasoning & Coding \\
        \midrule
        Actor learning rate & \(1e-6\) & \(1e-6\) \\
        Critic learning rate & \(5e-6\) & \(5e-6\) \\
        \midrule
        PACT actor update & DIS & PPO clipping \\
        PACT actor DIS range & \([0.7,6.0]\) & N/A \\
        PACT critic IS acceptance range & \([0,6]\) & \([0,6]\) \\
        \bottomrule
    \end{tabular}
\end{table*}

PPO, GRPO, and SAO use rollout log-probabilities for actor
updates, whereas PACT applies TIS with a ratio range of
\([0,2]\) to account for the mismatch between training and
rollout inference.
PACT w/o critic IS retains all other settings of PACT,
including actor-side corrections, the BCE critic objective,
and the Actor-then-Critic update order.

\section{Limitations}

Our characterization of credits does not imply that all possible notions of credit must take the form derived in this work. The uniqueness result is conditional on the three proposed regularity conditions. Just as replacing Euclid's parallel postulate leads to different but internally consistent geometries, adopting different requirements for credit may lead to different representations.

Second, the notion of credit studied here is statistical rather than causal. It is defined through conditional expectations and does not characterize the counterfactual causal effect of replacing an individual token or action.

Finally, this work characterizes what token-level credit should be under the proposed conditions, but does not solve the problem of estimating it exactly in practical LLM reinforcement learning. Developing accurate and efficient estimators of token-level credit remains an important direction for future work.

\end{document}